\documentclass[nohyperref]{article}

\usepackage{microtype}
\usepackage{graphicx}
\usepackage{booktabs} 
\usepackage{amsmath}

\usepackage{mathtools}

\usepackage{subcaption}
\usepackage{comment}

\usepackage{hyperref}

\usepackage[accepted]{icml2022}

\newif\ifshownotes
\shownotesfalse

\newif\ifarxiv
\arxivfalse
\usepackage{mylivemacros}
\usepackage{amsmath}
\usepackage{amssymb}
\usepackage{mathtools}
\usepackage{amsthm}

\usepackage[capitalize,noabbrev]{cleveref}
\usepackage{autonum}

\theoremstyle{plain}

\newtheorem{theorem}{Theorem}[section]
\newtheorem{proposition}[theorem]{Proposition}
\newtheorem{lemma}[theorem]{Lemma}

\theoremstyle{definition}

\theoremstyle{remark}
\newtheorem{remark}[theorem]{Remark}

\usepackage[textsize=tiny]{todonotes}

\icmltitlerunning{Fast rates for noisy interpolation requires
rethinking the effect of inductive bias}

\begin{document}

\twocolumn[
\icmltitle{Fast Rates for Noisy Interpolation Require Rethinking the Effects of Inductive Bias}

\icmlsetsymbol{equal}{*}

\begin{icmlauthorlist}
\icmlauthor{Konstantin Donhauser}{dinfk,aicenter}
\icmlauthor{Nicol\`o Ruggeri}{dinfk,mpi}
\icmlauthor{Stefan Stojanovic}{eth}
\icmlauthor{Fanny Yang}{dinfk}
\end{icmlauthorlist}

\icmlaffiliation{eth}{ETH Zurich}
\icmlaffiliation{dinfk}{ETH Zurich, Department of Computer Science}
\icmlaffiliation{aicenter}{ETH AI Center}
\icmlaffiliation{mpi}{Max-Planck-Institute for Intelligent Systems, T\"ubingen, Germany}

\icmlcorrespondingauthor{Konstantin Donhauser}{konstantin.donhauser@ai.ethz.ch}

\icmlkeywords{Machine Learning, ICML}

\vskip 0.3in]



\printAffiliationsAndNotice{}  

\begin{abstract}
Good generalization performance on high-dimensional data crucially hinges on a simple structure of the ground truth and a corresponding strong inductive bias of the estimator. Even though this intuition is valid for regularized models, in this paper we caution against a strong inductive bias for interpolation in the presence of noise: While a stronger inductive bias encourages a simpler structure that is more aligned  with the ground truth, it also increases the detrimental effect of noise. Specifically, for both linear regression and classification with a sparse ground truth, we prove that minimum $\ell_p$-norm and maximum $\ell_p$-margin interpolators achieve fast polynomial rates close to order $1/n$ for $p > 1$ compared to a logarithmic rate for $p = 1$. Finally, we provide preliminary experimental evidence that this trade-off may also play a crucial role in understanding non-linear interpolating models used in practice. 
\end{abstract}

\section{Introduction}
Despite being extremely overparameterized, large complex models such as deep convolutional neural networks generalize surprisingly well, even when interpolating noisy training data. 
If the noise in the training data is low,  a natural explanation could be that these models are biased towards having a certain structural simplicity. 
For example, in deep learning theory, a long line of work studies the \emph{implicit bias} of standard optimization algorithms towards solutions with a small structured norm, 
see e.g. \cite{telgarski_gradient_descent, Lyu2020Gradient, soudry18, chizat2020implicit,arora2019exact,jacot2018neural}.
If the optimal prediction model also has a small corresponding norm, then, intuitively, the implicit bias effectively reduces the search space
to a "good" subset that includes a good model.

Even though this intuition is valid in the low-noise regime, it is unclear why the generalization error might stay low when the (structured) models are forced to interpolate non-negligible noise in the data.
Further, it is hard to mathematically characterize structural simplicity for complex prediction models, and theoretical analysis is difficult.

Interestingly, some of the fundamental phenomena revolving around the generalization behavior of interpolating complex overparameterized models even occur for high-dimensional linear models (see e.g., \cite{hastie2019surprises,muthukumar_2020,bartlett_2020,deng21}  and references therein). Although the latter are significantly simpler to analyze, 
the literature has yet to 
provide a comprehensive theoretical understanding for interpolating models as it exists for example for regularized estimators. This paper aims to take an important step in that direction.

As mentioned above, high-dimensional parametric models with structural simplicity may correspond to parameter vectors with a small particular norm.
For linear models, the common structural simplicity assumption is sparsity, induced by the $\ell_0$/$\ell_1$-norms -- the concrete example we focus on in this paper.
We say that estimators with small $\ell_0$/$\ell_1$-norm 
have a \emph{strong inductive bias} towards ``simple'', in this case sparse, solutions. 
In contrast, the more frequently studied rotationally invariant $\ell_2$-norm
uniformly shrinks the estimator in all directions, thereby inducing only weak to no inductive bias towards sparse solutions.


For noiseless interpolation, it is well-known that the min-$\ell_1$-norm interpolator (aka basis pursuit) yields exact recovery for sparse ground truths \cite{chen_1998}. Moreover, when the measurements are noisy,
its regularized variant, the LASSO \cite{tibshirani_1996}, achieves minimax optimal rates of order $\frac{s\log(d)}{n}$
\cite{vandegeer_2008} (where $s$ is the $\ell_0$-norm of the ground truth).
However, when forcing structured models to interpolate the noisy samples, it is unclear why the generalization error might stay low. In fact, the min-$\ell_1$-norm interpolator achieves rates of order $\frac{1}{\log(d/n)}$ \cite{wang2021tight,muthukumar_2020}, suggesting that models with strong inductive biases suffer heavily from noise.


On the other hand, a long line of work establishes how 
min-$\ell_2$-norm interpolators with a weak inductive bias  benefit from noise resilience in high dimensions
(see e.g., \cite{hastie2019surprises,bartlett_2020,muthukumar_2020,muthukumar2021classification} and references therein). 
However, these uniform shrinkage estimators do not encode structural assumptions on the ground truth --- thus
fail to learn the signal  for
inherently high-dimensional covariates such as isotropic Gaussians (see e.g., \cite{muthukumar_2020}). 

The apparent trade-off between structural simplicity and noise resilience in high dimensions raises a natural question:
\begin{center}
\emph{Can min-norm interpolators in noisy high dimensional settings achieve fast or even close to minimax optimal rates using a moderate inductive bias?}
\end{center}

To the best of our knowledge, in this paper, we first provide a positive answer for sparse linear models. Specifically, we bound the (directional) estimation error of min-$\ell_p$-norm interpolators (for regression) and max-$\ell_p$-margin interpolators (for classification).
For isotropic Gaussian features with dimension $d \asymp n^\beta$ for $\beta>1$ and in the presence of observation noise in the data, 
\begin{itemize}
    \item we provide upper and lower bounds  for
    min-$\ell_p$-norm interpolators with ${p\in(1,2)}$. For large enough $d,n$, the estimation error decays at polynomial rates close to order $1/n$ compared to logarithmic or constant rates for $p=1$ or $2$ (Section \ref{sec:reg}). \vspace{-0.05in}
    \item we further provide upper bounds for the 
     max-$\ell_p$-margin interpolators (or equivalently, hard $\ell_p$-margin SVM) with ${p\in(1,2)}$. Surprisingly, for large enough $d,n$, they even match minimax optimal rates  of order $1/n$ up to logarithmic factors in the regime $\beta \geq 2$  (Section~\ref{sec:cla}). 
\end{itemize}

We confirm the better generalization properties for the choice $p\in (1,2)$ compared to $p=1$ and $p=2$ on synthetic and real-world data in Section~\ref{sec:discussion}.
Our results on linear models suggest that when interpolating noisy data, a moderate inductive bias yields the optimal performance. Additional experiments with convolutional neural tangent kernels in Section~\ref{sec:trade-off} provide  preliminary evidence that this intuition may also extend to non-linear models, prompting an exciting line of future work. In particular, we hypothesize that this trade-off between
structural simplicity and noise resilience may be an important ingredient for understanding the good generalization capabilities of overparameterized interpolating models used in practice.


\section{Minimum-norm Interpolators for Regression}
\label{sec:reg}
In this section we derive non-asymptotic bounds for sparse linear regression for the estimation error of min-$\ell_p$-norm interpolators with $p\in(1,2)$. We describe the setting in Section~\ref{subsec:regsetting} and present the main theorem followed by a discussion in Section~\ref{subsec:regmain}.


\subsection{Setting}
\label{subsec:regsetting}
We study a standard linear regression model where we observe $n$ pairs of standard normal distributed features $x_i \overset{\iid}{\sim} \NNN(0,I_d)$ and observations $y_i = \langle \xui, \wstar \rangle + \xi_i$ with Gaussian noise $\xi_i \overset{\iid}{\sim} \NNN(0,\sigma^2)$.  For simplicity,  
we consider the 1-sparse ground truth $\wgt = (1,0,\cdots,0)$ and discuss in Section~\ref{rm:p} how to generalize this assumption. 
Given the data set $\{\left(\xui, \yui\right)\}_{i=1}^n$, the goal is to find an estimator $\what$ that has small estimation error
\begin{equation}
    \label{eq:prediction_error_reg}
 \RiskR(\what) := \norm{\what -\wgt}_2^2 = \EE_{x\sim \NNN(0,I_d)} \langle x, \what - \wgt \rangle^2    \, ,
\end{equation}
which also corresponds to the irreducible prediction error. We specifically study \emph{min-$\ell_p$-norm interpolators} with $p\in (1,2)$, given by
      \begin{equation}\label{eq:minlpnomrinterpol}
    \what = \argmin_{w} \norm{w}_p \subjto \forall i:~\langle \xui, w \rangle = y_i, 
\end{equation}
where $\|w\|_p := \left(\sum_i w_i^p\right)^{1/p}$. 

\subsection{Main Result}
\label{subsec:regmain}
We now state our main result for regression that provides a  non-asymptotic upper bound for the estimation error of min-$\ell_p$-norm interpolators.\footnote{
We use $\lesssim,\gtrsim$ and $\asymp$ to hide universal constants, without any hidden dependence on $d$, $n$ or $p$. Further, $\tilde{O}(\cdot),\tilde{\Theta}(.)$ hide logarithmic factors in $d,n$ or $p$ and $a \lor b = \max(a,b)$.
}
\begin{theorem}
\label{thm:regressionlpmain}
Let the data distribution be as described in Section~\ref{subsec:regsetting} and assume that $\sigma \asymp  1 $. Further, let $q$ be such that $\frac{1}{p} + \frac{1}{q} =1$.  Then there exist universal constants $\kappa_1,\cdots,\kappa_7>0$ such that for any $n\geq \kappa_1$ and any ${p \in\left(1+\frac{\kappa_2}{\log\log d},2\right)}$ and ${n \log^{\kappa_3} n \lesssim d \lesssim n^{q/2}\log^{-\kappa_4 q}n}$, the estimation error of the min-$\ell_p$-norm interpolator~\eqref{eq:minlpnomrinterpol} satisfies
\ifarxiv
\begin{align}
\label{eq:thmregboundhighnoiseuppermain}
      &\frac{\sigma^{4-2p}q^pd^{2p-2}}{ n^p} 
    \lor \frac{\sigma^2n}{d} \lesssim
    \RiskR\left( \what \right) \lesssim  \frac{\sigma^{4-2p}q^p d^{2p-2}}{n^p} 
    \lor \frac{\sigma^2n  \exp(\kappa_5 q)}{q d} \, ,
\end{align}
\else
\begin{align}
\label{eq:thmregboundhighnoiseuppermain}
      &\RiskR\left( \what \right) \gtrsim \frac{\sigma^{4-2p}q^pd^{2p-2}}{ n^p} 
    \lor \frac{\sigma^2n}{d} ~~\mathrm{and}\\
    &\RiskR\left( \what \right) \lesssim  \frac{\sigma^{4-2p}q^p d^{2p-2}}{n^p} 
    \lor \frac{\sigma^2n  \exp(\kappa_5 q)}{q d} \, , \nonumber
\end{align}
\fi
with probability at least $ 1 - \kappa_6 d^{-\kappa_7}$ over the draws of the data set. 
\end{theorem}
The proof of the theorem can be found in Appendix~\ref{apx:regproof}.  We now discuss the implications of the theorem and  refer to Section~\ref{rm:p} for a discussion on the assumptions in the theorem. Throughout the discussion we consider the regime $d \asymp n^{\beta}$ with $\beta>1$.

\paragraph{Close to  minimax optimal rates.}
Theorem~\ref{thm:regressionlpmain} implies that for any fixed $\beta>1$, the estimation error $\Risk(\what) = \tilde{O}(\alpha)$ of the min-$\ell_p$-norm interpolator with ${p \in \left(1 + \frac{\kappa_2}{\log\log d}, \frac{ 2\beta}{2 \beta - 1}\right)}$ vanishes at a polynomial rate ($\alpha<0$). To illustrate these rates, Figure~\ref{fig:rates_reg} plots the exponent of the rate $\alpha$ as a function of $\beta$ for different values of $p$, assuming that $d,n$ are sufficiently large. We compare the rates against the minimax optimal rate for sparse regression  of order $\frac{\sigma^2 \log d}{n}$ (dotted horizontal line at $\alpha = -1$) \cite{Raskutti2011MinimaxRO}.
We can clearly see in Figure~\ref{fig:rates_reg} that for $\beta>2$ and small values of $p$, the rates of the error are close to the minimax optimal rate. In fact, when choosing $\beta =2$ and $p = 1 + \frac{ \kappa_4}{\log\log d}$, Theorem~\ref{thm:regressionlpmain} shows that, for $n,d$ sufficiently large, the rate of the error equals the minimax optimal rate up to logarithmic factors.

\paragraph{Faster rates than for $p=1$}
For comparison, we also indicate the rates of the min-$\ell_2$- and min-$\ell_1$-norm interpolators by the dashed horizontal line at $\alpha=0$, which are of constant and  $\frac{\sigma^2}{(\beta-1) \log n}$ order, respectively, \cite{wang2021tight,muthukumar_2020}. Clearly, we can see that the minimum-$\ell_p$-norm interpolator with $p\in(1,2)$ achieves faster rates than with $p=1$ and $2$. 
We emphasize that Figure~\ref{fig:rates_reg} depicts the exponent of the rate at which the error vanishes up to logarithmic factors for fixed values of $p$ as $d,n \to \infty$. For fixed $n,d$, our non-asymptotic bounds only hold for $p \in (1+\frac{\kappa}{\log\log d}, 2)$. We refer to Section~\ref{rm:p} for a discussion of this limitation and motivate future work on tight bounds for the full interval $p\in[1,2]$  in Section~\ref{subsec:optimalp}.

\begin{figure}[t!]
    \centering
         \centering
         \ifarxiv
    \includegraphics[width=0.4\linewidth]{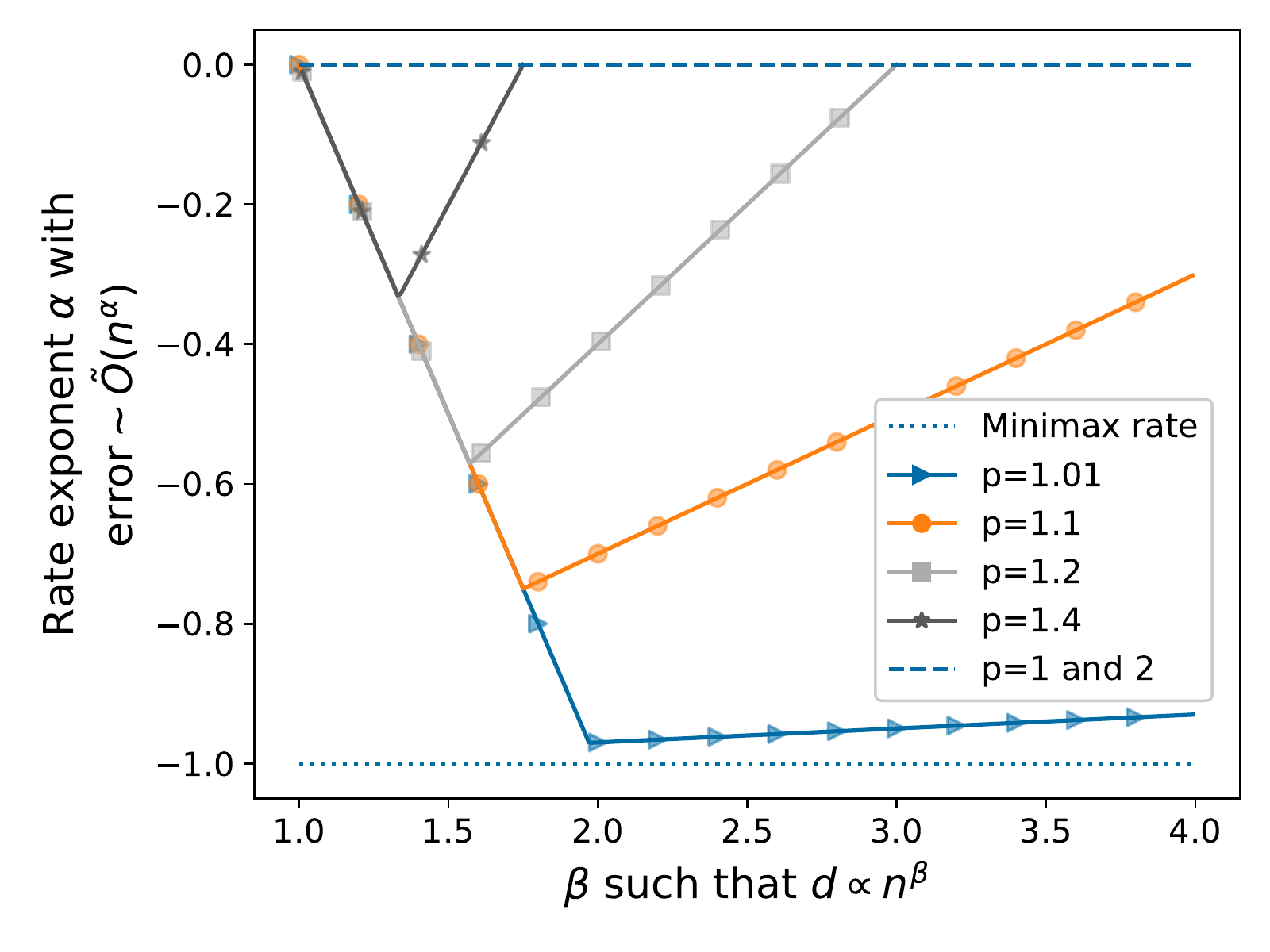}
    \else 
        \includegraphics[width=\linewidth]{figures/regression_rates.pdf}
\fi
        \caption{Depiction of Theorem~\ref{thm:regressionlpmain} for regression
        when $d\asymp n^{\beta}$. We plot the exponent $\alpha$ of the resulting estimation error rate 
        $\tilde{O}(n^{\alpha})$ for different strengths of inductive bias $p$ at different high-dimensional regimes $\beta$.}
    \label{fig:rates_reg}
\end{figure}

\paragraph{Tightness of the upper and lower bounds in Theorem~\ref{thm:regressionlpmain} } 
We note that for any fixed $p$ and $n,d$ sufficiently large, the upper and lower bounds in Theorem~\ref{thm:regressionlpmain} are of the same order. In fact, 
the second term in the lower bound of order $\frac{n}{d}$ is a universal lower bound for all interpolators  \cite{muthukumar_2020}. This bound is tight when the term $\frac{\sigma^2n  \exp(\kappa_5 q)}{q d} $  dominates the upper bound in Theorem~\ref{thm:regressionlpmain}, which is the case when  $\beta\leq 2$ and $p$ is small constant.





\paragraph{Comparison with existing bounds for min-$\ell_p$-norm interpolators.}
We now discuss existing results for min-$\ell_p$-norm interpolators from the literature. 
To the best of our knowledge,  previous works do not  study the rates of the  min-$\ell_p$-norm interpolator. However, we may obtain an upper bound
as a consequence of Theorem~4 in \cite{koehler2021uniform} which follows straightforwardly from applying Lemma~\ref{cor:hqbound2} in Appendix~\ref{apx:subsec:localisationproofreg}
  and Lemma~\ref{cor:hqbound1} in Appendix~\ref{apx:proofdualnormconc}.

\begin{theorem}[Corollary of Theorem~4 from
\cite{koehler2021uniform} (informal)]
\label{cr:koehler}
Let the data distribution be as described in Section~\ref{subsec:regsetting} with general ground truth $\wgt\in\RR^d$. Under the same conditions on $d,n$ as in Theorem~\ref{thm:regressionlpmain}, there exist universal constants $\kappa_1,\kappa_2,\kappa_3>0$ such that for any $p\in (1,2]$ and $q$ such that $ \frac{1}{p} + \frac{1}{q} =1 $, we have that 
\ifarxiv
\begin{align}
    \RiskR(\what) &\lesssim \sigma^2 \left( \sqrt{\frac{\log d}{q}} \frac{1}{ d^{1/q}} + \sqrt{\frac{\log d}{n}}
    + \frac{n\exp(\kappa_1 q)}{qd} \right) + \sigma \norm{\wgt}_p d^{1/q} \sqrt{\frac{q}{n}} + \norm{\wgt}_p^2 \frac{ q d^{2/q}}{n} \nonumber
\end{align}
\else 
\begin{align}
    \RiskR(\what) &\lesssim \sigma^2 \left( \frac{\sqrt{\log d}}{\sqrt{q} d^{1/q}} + \sqrt{\frac{\log d}{n}}
    + \frac{n\exp(\kappa_1 q)}{qd} \right) \\
    &~~~~~+ \sigma \norm{\wgt}_p d^{1/q} \sqrt{\frac{q}{n}} + \norm{\wgt}_p^2 \frac{ q d^{2/q}}{n} \nonumber
\end{align}
\fi
with probability at least $1- \kappa_2 d^{-\kappa_3}$ over the draws of the data set.
\end{theorem}

We remark that it already follows from Theorem~\ref{cr:koehler}, in combination with the results in \cite{wang2021tight},  that $p>1$ achieves faster rates compared to $p=1$. However, for any $p$ and $\beta>1$ with $d=n^{\beta}$ the rates in the upper bound in Theorem~\ref{cr:koehler} are at most of order $n^{-1/4}$, and thus much slower than the rates in Theorem~\ref{thm:regressionlpmain} that can reach orders even close to $n^{-1}$.  
For a better comparison, we illustrate both the rates in Theorem~\ref{thm:regressionlpmain} and Theorem~\ref{cr:koehler} in Figure~\ref{fig:ratesk} in Appendix~\ref{rm:koehler}, which also provides a detailed comparison of the proof techniques used to derive Theorems~\ref{thm:regressionlpmain} and~\ref{cr:koehler}.

\paragraph{Comparison with $\ell_p$-norm regularized estimators}
We are only aware of existing bounds in the literature for $\ell_p$-norm regularized estimators that characterize the ground truth by the  $\ell_p$-norm (e.g. see \cite{lecue17}). These bounds are at least of order $\|\wgt\|_p \frac{d^{2 - 2/p}}{n}$ (see Theorem 5.4 in  \cite{lecue17}), and thus slower than the bounds for the corresponding interpolating estimator in Theorem~\ref{thm:regressionlpmain}, for some choices of $\beta$.
Furthermore, we hypothesize when the term  $\frac{\sigma^{4-2p}q^p d^{2p-2}}{n^p}$ on the RHS in Equation~\eqref{eq:thmregboundhighnoiseuppermain} dominates, the optimally $\ell_p$-norm regularized estimator achieves the same rates as the corresponding interpolator. As can be seen in the proof of Theorem~\ref{thm:regressionlpmain} in Appendix~\ref{apx:regproof}, this term captures the error that arises from the orthogonal projection of $\what$ onto the direction of the signal $\wstar$. This error is expected to increase when adding explicitly regularization, and thus shrinking the estimator.

\section{Maximum-margin Interpolators for Classification}
\label{sec:cla}
We now establish upper bounds for max-$\ell_p$-margin interpolators, also called hard-margin $\ell_p$-SVMs or sparse-SVMs \cite{blanco2020lp,bennett2000duality}.
We show that these interpolators achieve fast rates for $p\in(1,2)$ and even match minimax optimal rates up to logarithmic factors.

\subsection{Setting}
\label{subsec:clasetting}
We study a discriminative linear classification setting with $n$ pairs of random input features $x_i \overset{\iid}{\sim} \NNN(0,I_d)$ and associated labels 
$y_i =\sgn(\langle \xui, \wgt \rangle)\xi_i$
where the label noise 
$\xi_i \in \{+1, -1\}$ follows the conditional distribution 
\begin{equation}
    \xi_i | x_i \sim \probsigma( \cdot; \langle x_i, \wgt \rangle)
\end{equation}
for some parameter~$\sigma$.
We again choose $\wgt = (1,0,\cdots,0)$ for the same reason as in Section~\ref{subsec:regsetting}.
Notice that the noise only depends on the input features in the direction of the ground truth.
More specifically,  we make the 
following assumption on the noise distribution~$\probsigma$:
\begin{itemize}
  \item[]
    The function ${z \to \probsigma(\xi=1; z)}$ is a piece-wise continuous function such that the minimum ${\nubar := \underset{\nu}{\arg\min}~ \EE_{Z \sim \NNN(0,1)} \EE_{\xi \sim \probsigma(\cdot; Z)} \left( 1- \xi \nu \vert Z \vert\right)_+^2 }$ exists and is positive $\nubar>0$.  
 \end{itemize}

Assumption A is rather weak and satisfied by most noise models in the literature,  such as
\begin{itemize}
    \item\textit{Logistic regression} with  $\probsigma(\xi_i=1; z) = h(z \sigma)$ and $h(z) = \frac{e^{\vert z\vert}}{1+e^{\vert z\vert}}$ and $\sigma >0$. 
    \item \textit{Random label flips} with  $\probsigma(\xi=1; \langle \xui, \wstar \rangle) = 1-\sigma$ and  $\sigma \in (0,\frac{1}{2})$. 
    \item \textit{Random noise before quantization} where $\yui = \sgn(\langle \wgt, \xui \rangle + \tilde{\xi}_i)$ with $\tilde{\xi}_i\vert x_i \sim \NNN(0, \sigma^2)$ and  $\sigma^2>0$.
\end{itemize}

Given the data set $\{\left(\xui, \yui\right)\}_{i=1}^n$, the goal is to obtain an estimate $\what$ that
directionally aligns with the  normalized ground truth $\wgt$ and thus has a small directional estimation error
\begin{align}
\label{eq:prediction_error_cla}
    \RiskC(\what) := \norm{\frac{\what}{\norm{\what}_2} - \wgt}_2^2.
\end{align}
This  classification error is also studied for example in the 1-bit compressed sensing literature (see e.g., \cite{Boufounos2008,plan2012robust} and references therein). 
Note that a small value $\RiskC(\what)$ corresponds to a small expected (noiseless)  0-1 error via the relation
\ifarxiv
\begin{align}
  &\EE_{x\sim \NNN(0,I_d)} \idvec[\sgn(\langle x, \what\rangle) \neq \sgn(\langle x,\wgt \rangle)] \nonumber = \frac{1}{\pi} \arccos\left(1- \frac{\RiskC(\what)}{2}  \right) \approx \frac{1}{\pi} \sqrt{\frac{\RiskC(\what)}{2}} \, .
\end{align}
\else
\begin{align}
  &\EE_{x\sim \NNN(0,I_d)} \idvec[\sgn(\langle x, \what\rangle) \neq \sgn(\langle x,\wgt \rangle)] \nonumber\\
    &= \frac{1}{\pi} \arccos\left(1- \frac{\RiskC(\what)}{2}  \right) \approx \frac{1}{\pi} \sqrt{\RiskC(\what)} \, .
\end{align}

\fi

Throughout this section, we study the \emph{max-$\ell_p$-margin interpolators}, or equivalently, the hard-margin $\ell_p$-SVM solutions  for $p\in [1,2]$ defined by 
      \begin{equation}\label{eq:svm problem}
    \what = \argmin_{w} \norm{w}_p \subjto \forall i:~ \yui \langle \xui, w \rangle \geq 1.
\end{equation}

\subsection{Maximum-\texorpdfstring{$\ell_p$}{lp}-margin Interpolation} 
\label{subsec:clalp}
We are the first to present non-asymptotic upper bounds for the directional estimation error of max-$\ell_p$-margin interpolators.

\begin{theorem}
\label{thm:classificationlp}
    Let the data distribution be as described in Section~\ref{subsec:clasetting} and assume that the noise model $\probsigma$ is independent of $n,d$ and $p$.
Let $q$ be such that $\frac{1}{p} + \frac{1}{q} =1$. 
There exist universal constants 
$\kappa_1,\cdots,\kappa_7>0$
such that for any $n\geq \kappa_1$, any $p \in\left(1+\frac{\kappa_2}{\log\log d},2\right)$ and any
$ n \log^{\kappa_3}n \lesssim d \lesssim \frac{n^{q/2}}{\log^{\kappa_4 q}n}$, the directional estimation error of the max-$\ell_p$-norm interpolator~\eqref{eq:svm problem} satisfies
\begin{equation}
\label{eq:classificationlp}
\RiskC(\what) \lesssim  \frac{ q^{\frac{3}{2}p}d^{3p-3} \log^{3/2}d}{n^{\frac{3}{2}p} } \lor \frac{n  \exp(\kappa_4q)}{qd} \lor  \frac{\log^{\kappa_5} d}{n} \, ,
\end{equation}
with probability at least $ 1 - \kappa_6 d^{-\kappa_7}$ over the draws of the data set. 
\end{theorem}
The proof of the result is presented in Appendix~\ref{apx:claslpproof}.
The dependence on the noise  $\probsigma$ is hidden in  the universal constants but can be made explicit when carefully following the steps in the proof.   We refer to Section~\ref{rm:p} for a discussion on the assumptions in the theorem. Throughout the discussion we consider the regime $d \asymp n^{\beta}$ with $\beta>1$.

\begin{figure}[t!]
    \centering
    \ifarxiv
    \includegraphics[width=0.4\linewidth]{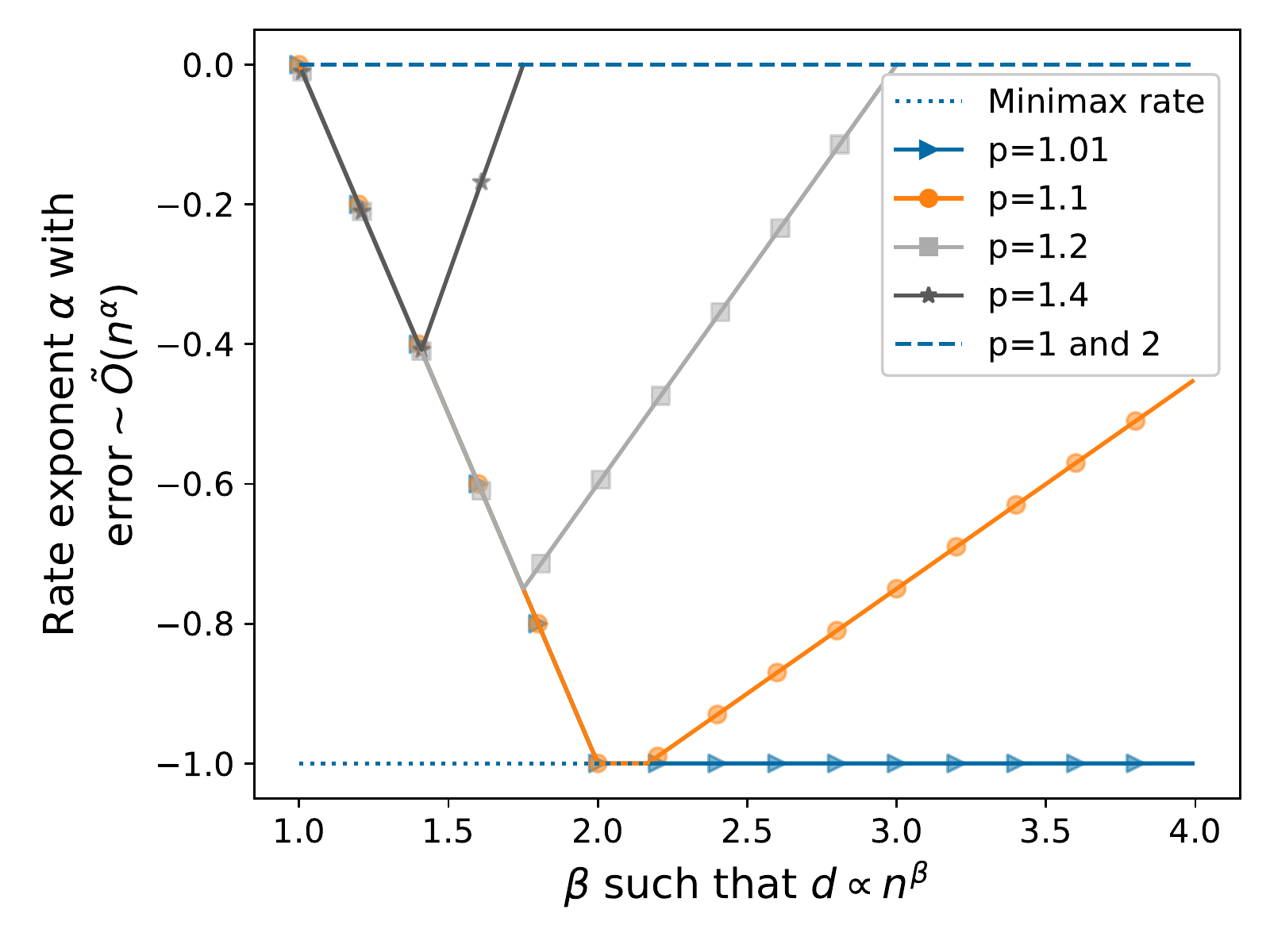}
    \else 
        \includegraphics[width=\linewidth]{figures/classification_rates.pdf}
    \fi
         \caption{Depiction of Theorem~\ref{thm:classificationlp} for classification  when $d\asymp n^{\beta}$. We plot the exponent $\alpha$ of the resulting  upper bound of the (directional) estimation error rate 
        $\tilde{O}(n^{\alpha})$  for different strengths of inductive bias $p$ and different high-dimensional regimes $\beta$.}
    \label{fig:ratescla}
\end{figure}

\paragraph{Close to  minimax optimal rates}
As for regression, Theorem~\ref{thm:classificationlp} 
implies that 
 the estimation error $\Risk(\what) = \tilde{O}(\alpha)$ of the max-$\ell_p$-margin interpolator with ${p \in \left(1 + \frac{\kappa_2}{\log\log d}, \frac{3 \beta}{3 \beta - 1.5}\right)}$ vanishes at a polynomial rate ($\alpha<0$).  We again illustrate the rates in Figure~\ref{fig:ratescla} where we plot the exponent of the rate $\alpha$ as a function of $\beta$ for different values of $p$, assuming that $d,n$ is sufficiently large. Furthermore, we plot the minimax optimal lower bounds  (dotted line at $\alpha=-1$) for the directional estimation error $\RiskC$, which are known to be of order $\frac{\log d}{n}$  \cite{Wainwright2009InformationTheoreticLO,abramovich2018high}. 
When $\beta>2$ and $p$ is small, the rates of the max-$\ell_p$-norm interpolator
are of order $\tilde{O}(\frac{1}{n})$ and hence minimax optimal up to logarithmic factors. 
More specifically, for any  $\beta>2$ and $p\in \left( 1 + \frac{\kappa_2}{\log\log d}, \frac{6\beta-2}{6\beta - 3} \right)$, the term $ \frac{\log^{\kappa_5}d}{n}$ dominates the RHS in Equation~\eqref{eq:classificationlp} and we obtain a minimax optimal rate up to logarithmic factors.



\vspace{-1mm}
\paragraph{Faster rates than $p=1$}
So far, previous non-asymptotic upper bounds for the error of the max-$\ell_1$-margin interpolator  \cite{chinot2021adaboost}  are non-vanishing, while to the best of our knowledge we are not aware of any lower bounds. However, using the same tools as used for bounding the error of the min-$\ell_1$-norm interpolator in \cite{wang2021tight} and the tools introduced in the proof of Theorem~\ref{thm:classificationlp}, we can upper and lower bound\footnote{The explicit theorem statement is moved from the camera ready version of this paper as a response to the reviewers concerns on the length of the paper. The statement will instead appear in a followup work of \cite{wang2021tight}.} the directional estimation error \eqref{eq:prediction_error_cla} by  $\RiskC(\what) \asymp \frac{1}{\log(d/n)}$.

Analogous to regression, Figure~\ref{fig:ratescla}  shows how the max-$\ell_p$-margin interpolators with $p>1$ achieve faster rates compared to $p =1$ and $2$ (dashed line at $\alpha=0$ in Figure~\ref{fig:ratescla}).
More specifically, for any fixed $\beta>1$, we obtain  faster rates  with any $p \in \left(1 + \frac{\kappa_2}{\log\log d}, \frac{3 \beta}{3 \beta - 1.5}\right)$ than with $p=1$ and $2$.

\paragraph{Lower bounds} Unlike Theorem~\ref{thm:regressionlpmain}, Theorem~\ref{thm:classificationlp} does not provide matching lower bounds.
However, when $\beta \geq 2$ and $p$ is small, the rates in Theorem~\ref{thm:classificationlp} are close to minimax lower bounds and therefore cannot be improved. Furthermore, 
similar to regression, in Proposition~\ref{prop:uniformlowerbound} we provide  a uniform lower bound of order $\frac{n}{d}$ for all interpolating classifiers that matches the upper bound  in Theorem~\ref{thm:classificationlp}  for the regime $\beta\leq 2$ and when $p$ is small.  
 The proof can be found in Appendix~\ref{apx:uniformlowerbound}. 
\begin{proposition}[Universal lower bound for all interpolating classifiers]
\label{prop:uniformlowerbound}
  Let the data distribution be as described in Section~\ref{subsec:clasetting} and assume that the noise model $\probsigma$ is independent of $n,d$ and $p$.
  There exist universal constants $\kappa_1,\kappa_2>0$ such that 
for all non-zero ground truths $\wgt\in\mathbb{R}^d$, with probability $\geq 1-\exp(-\kappa_1 n) -\exp(-\kappa_2 d)$ for some $\kappa>0$,
the directional estimation error of any interpolator $\what$ satisfying $\forall i: \yui \langle \xui, \what \rangle \geq 0$ satisfies
\begin{equation}
    \RiskC(\what) \geq \frac{c n}{d}. 
\end{equation}
\end{proposition}

\paragraph{Comparison with regression.}
Comparing the rates in Figures~\ref{fig:rates_reg} and \ref{fig:ratescla}, we can see that the max-$\ell_p$-margin interpolators in classification achieve faster rates than the corresponding min-$\ell_p$-norm interpolators in regression, in the sense that they are of order of minimax lower bounds for a wider range of $\beta$ and $p\in(1,2)$. 
We remark that the proofs of Theorems~\ref{thm:regressionlpmain} and \ref{thm:classificationlp} follow a similar scheme (see
detailed discussion in Section~\ref{sec:proofidea}).
Intuitively,
the difference in performance mostly originates from the fact that the  directional estimation error $\RiskC$ merely depends on the direction of the interpolator $\what$ and not its magnitude,
as it is the case for the estimation error $\RiskR$ in regression. We note that the authors in \cite{muthukumar2021classification} observe a similar difference in performance between max-margin and min-norm interpolators for $p=2$
and Gaussians features with spiked covariance matrices.

 \begin{figure*}
    \centering
\begin{subfigure}[b]{0.32\linewidth}
         \centering
         \includegraphics[width=\textwidth]{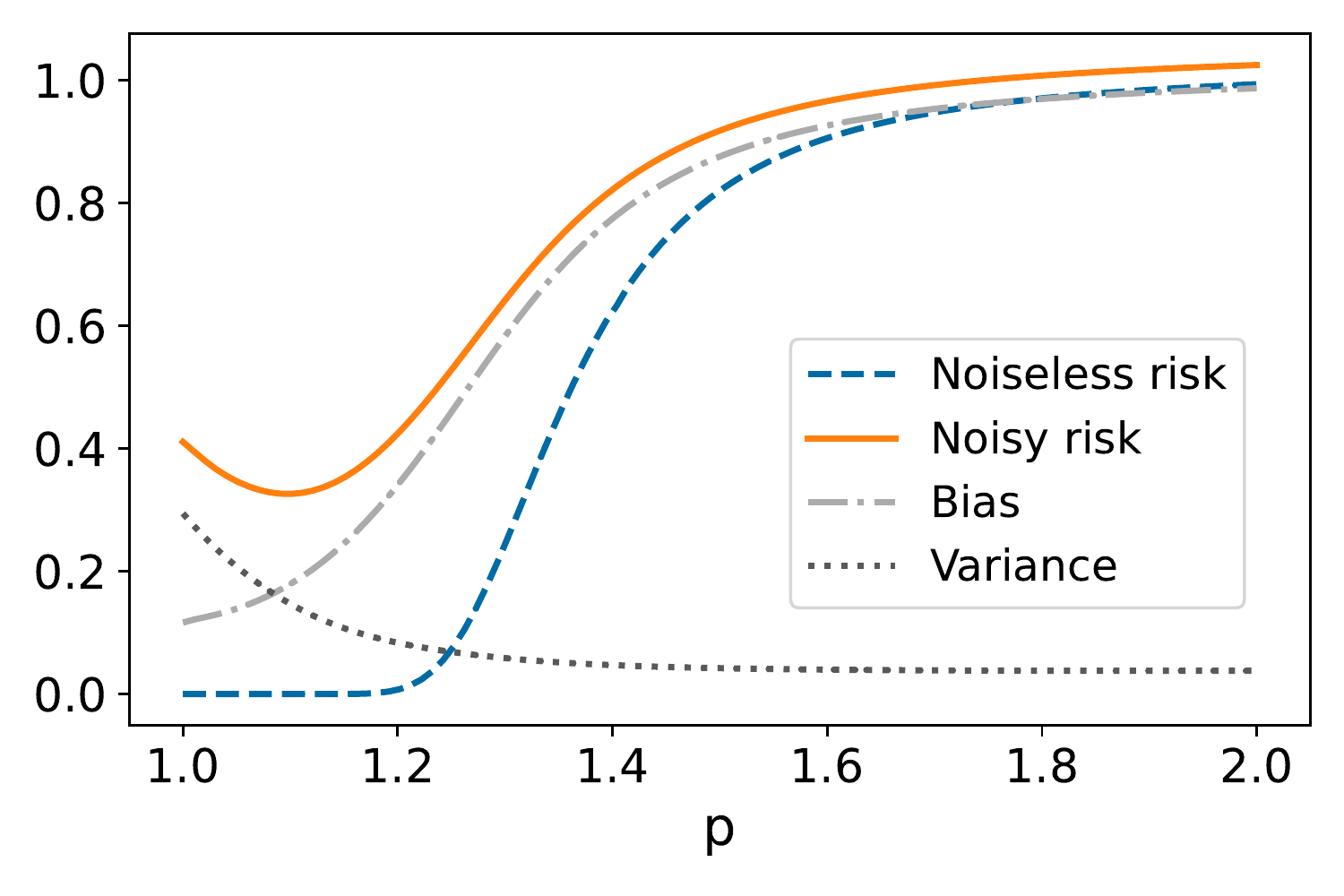}
         \caption{ New bias-variance trade-off}
         \label{fig:synthetic_regression_bias_variance}
     \end{subfigure}
     \begin{subfigure}[b]{0.32\linewidth}
         \centering
   \includegraphics[width=\textwidth]{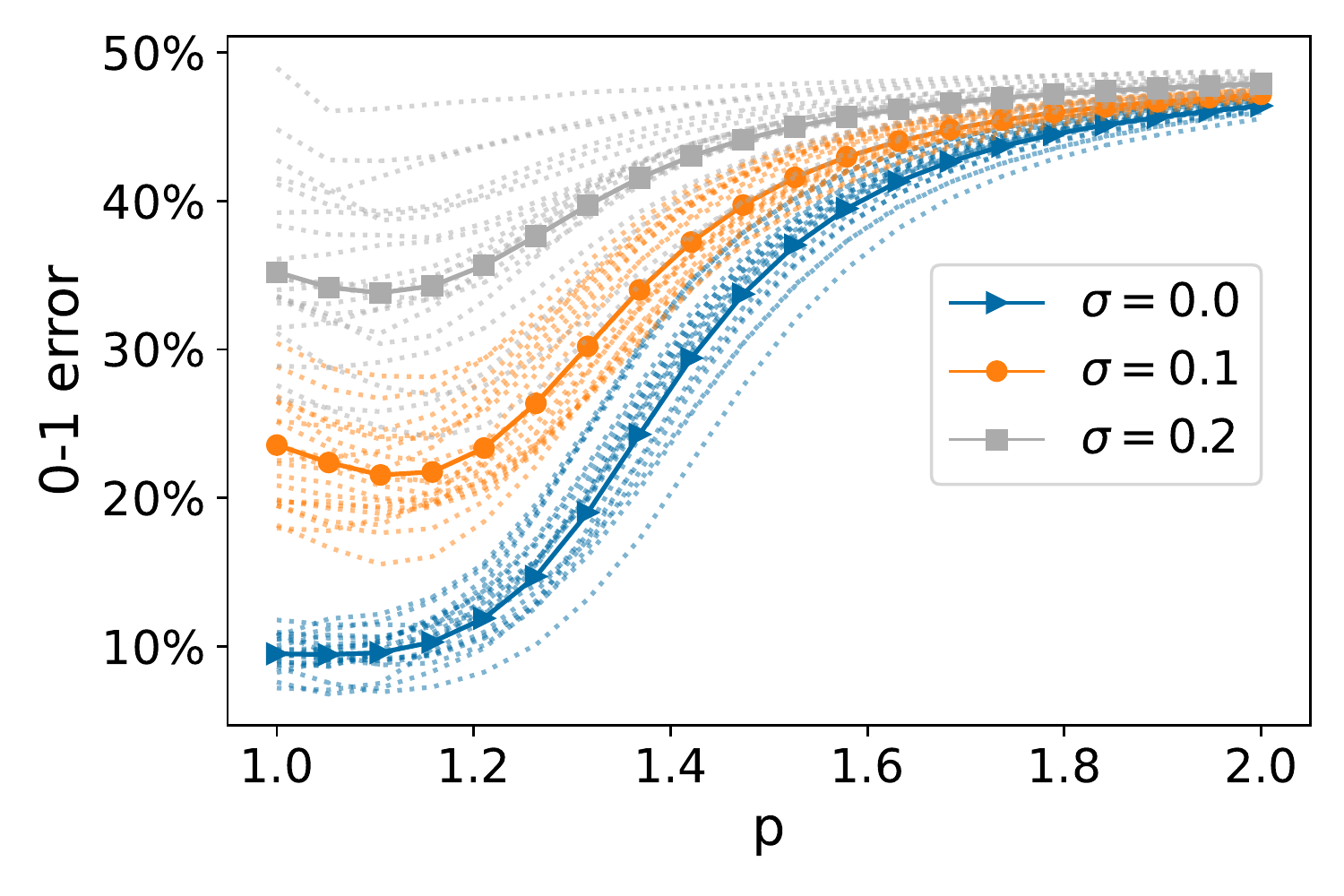}
         \caption{Simulations Classification}
         \label{fig:synthetic classification subplot}
     \end{subfigure}
\begin{subfigure}[b]{0.32\linewidth}
         \centering
         \includegraphics[width=\textwidth]{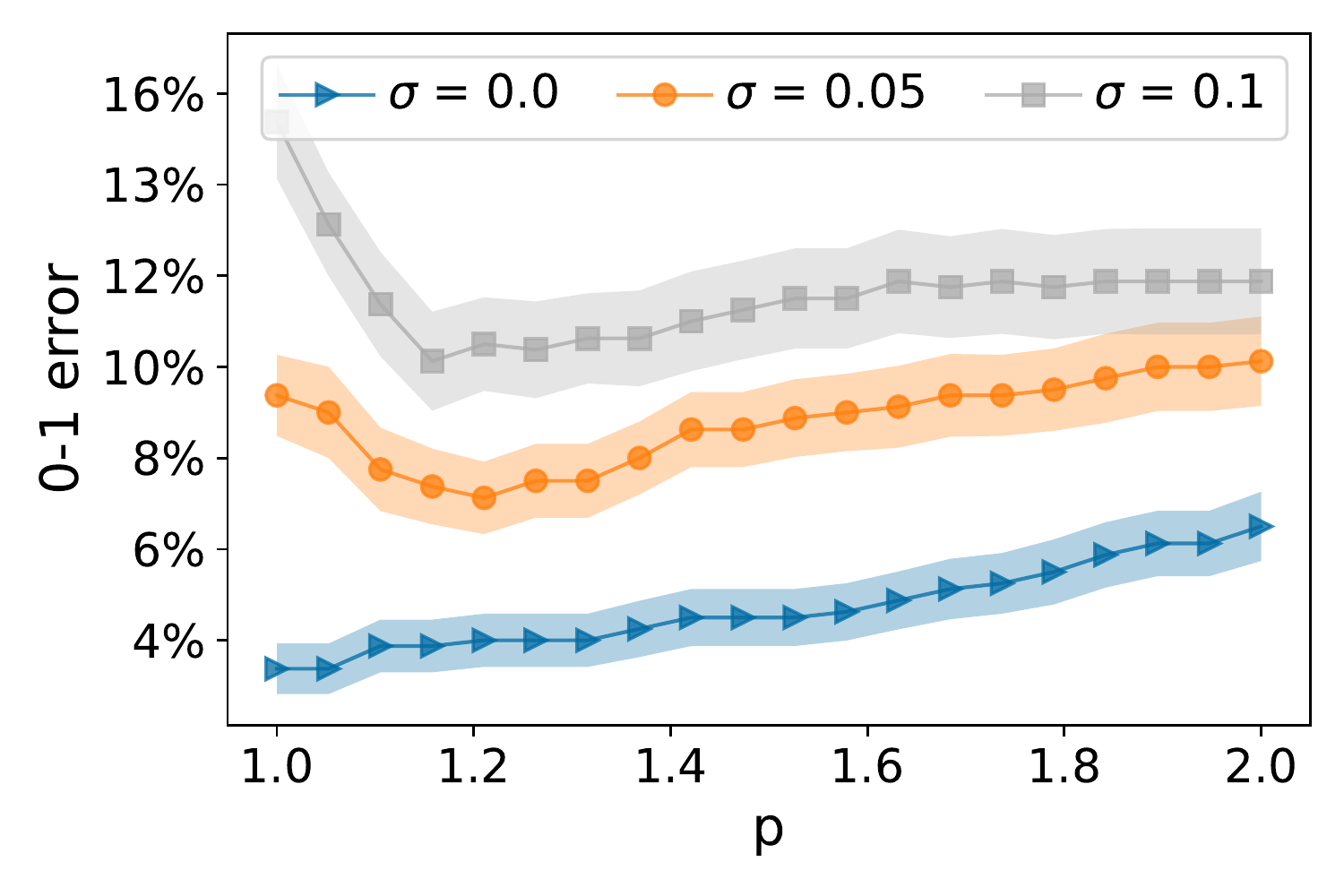}
         \caption{Leukemia classification dataset}
         \label{fig:leukemia classification subplot}
     \end{subfigure} 
         \caption{
         \textbf{(a)} The averaged estimation error (risk) as well as the estimated bias and variance of the min-$\ell_p$-norm interpolator \eqref{eq:minlpnomrinterpol} as a function of $p$. We run 50 independent simulations with data drawn from the distribution in Section~\ref{subsec:regsetting} and choose $n= 100,d=5000$, $\sigma =1$ (noisy) and $\sigma = 0$ (noiseless).
         \textbf{(b)} Average and individual classification error for the max-$\ell_p$-margin interpolator \eqref{eq:svm problem} with $n=100$, $d=5000$ over 50 independent runs. 
         \textbf{(c)}
         Mean and variance of the classification error for the max-$\ell_p$-margin interpolator \eqref{eq:svm problem}
          on the Leukemia data set with $d=7070$ and $72$ data points in total. See Section~\ref{sec:discussion} for further details.}
    \label{fig:classification svm}
\end{figure*}

\section{Numerical Simulations and Interpretation of the Results}
\label{sec:discussion}

We now present empirical evidence for our theory\footnote{We release the code for the replication of the experiments at \url{https://github.com/nickruggeri/fast-rates-for-noisy-interpolation}}, 
which predicts the superiority of min-$\ell_p$-norm/max-$\ell_p$-margin interpolators with moderate inductive bias $p\in(1,2)$ compared to $p=1,2$.
Further, we discuss how varying the strength of the inductive bias induces a trade-off between 
structural simplicity aligning with the ground truth and noise resilience in high dimensions.
Finally, we show experimentally that this trade-off can also be intuitively  understood as a "new" bias-variance trade-off induced by varying the inductive bias.

\vspace{-1mm}
\paragraph{Simulations for regression and classification}
While Theorems~\ref{thm:regressionlpmain} and \ref{thm:classificationlp} hold for "large enough" $d, n$, we now illustrate experimentally in Figure~\ref{fig:synthetic_regression_bias_variance} and \ref{fig:synthetic classification subplot} that the superiority of the choice $p\in(1,2)$ over $p=1$ and $p=2$ also holds for finite $d,n$ of practically relevant orders:
For $d=5000,n=100$, the min-$\ell_p$-norm/max-$\ell_p$-margin interpolators with $p\in(1,2)$ achieve lowest (directional) estimation error in the noisy case ($\sigma >0$), while $p=1$ is optimal in the noiseless case ($\sigma = 0$).
Both experiments are run on a synthetic dataset with a 1-sparse ground truth $\wstar$, as described in Section~\ref{subsec:regsetting} (regression) and Section~\ref{subsec:clasetting} (classification). For the regression experiment, we plot the average risks over 50 runs for $\sigma =1$ (noisy) and $\sigma = 0$ (noiseless) and the bias and variance in the noisy case. 
For the classification experiment, we randomly flip a fixed amount of $\sigma\%$ of the labels.
 We plot the error of individual random draws and their mean over 50 independent runs.

Furthermore, we examine the performance of the max-$\ell_p$-margin classifier on the Leukemia classification dataset \cite{leukemiapaper} with $d=7070$.
In line with our theory, Figure~\ref{fig:classification svm} shows
that the value of $p$ which minimizes the (directional) estimation error shifts from $p=1$, in the noiseless case,  to $p\in(1,2)$ in the presence random label flips in the data. Noiseless, i.e. $\sigma=0$, here means that we are not adding artificial label noise.   We plot the  averaged error and its variance over $100$ random train-test splits with training set size $n=65$. 
Finally, note that in practical applications with noisy data, we do not advocate  the use of interpolators but recommend using regularized estimators instead. 

\vspace{-1mm}
\paragraph{The strength of the inductive bias induces a new bias-variance trade-off.}
Why the optimal choice of $p$ does not  correspond to the strongest inductive bias ($p=1)$  can be explained by a
trade-off between two competing factors:
the regularizing effect of high-dimensionality and the effective sparsity of the solutions, matching the simple structure of the ground truth. In particular, with increasing inductive bias (decreasing $p\to 1$), noise resilience due to high dimensions decreases while the interpolator becomes effectively sparser and hence achieves better performance on noiseless data. This trade-off is also reflected  in Theorems~\ref{thm:regressionlpmain} and \ref{thm:classificationlp}: The terms of order $\frac{\sigma^2 n  \exp(\kappa_4 q)}{qd}$ 
capture the noise resilience due to high-dimensionality, or rather vulnerability, and monotonically decrease with $d$ but increase as $p\to 1$.
On the other hand, the terms of order $\frac{ q^p\sigma^{4-2p}d^{2p-2}}{n^p}$ and $\frac{q^{\frac{3}{2}p}d^{3p-3} \log^{3/2}d }{n^{\frac{3}{2}p}}$
capture the benefits of
the structural alignment with the sparse ground truth as they decrease with $p\to1$.

We now give a more intuitive reasoning for this trade-off,  which may also translate to more general models. Clearly, to recover noiseless signals in high dimensions, the space of possible solutions must be restricted by using an inductive
bias that encourages the structure of the interpolator to match that of the ground truth.
For instance, the solution of the min-$\ell_1$-norm interpolator is always $n$-sparse despite having $d$ parameters. 
However, exactly  this restriction  towards a certain structure (such as sparsity)  becomes harmful
when fitting noisy labels: Instead of low $\ell_2$-norm 
solutions that can distribute noise across all dimensions, 
the interpolator is forced to find $n$-sparse 
solutions with a higher $\ell_2$-norm.

In Figure~\ref{fig:synthetic_regression_bias_variance}, we demonstrate how
the trade-off 
can also be viewed as a novel kind of bias-variance trade-off for interpolating models. 
In the classical bias-variance trade-off, increasing  model complexity (e.g., by decreasing the regularization penalty) leads to larger variance but smaller (statistical) bias. For interpolating models, we observe a similar trade-off when varying the strength of the inductive.  Figure~\ref{fig:synthetic_regression_bias_variance} depicts this trade-off and shows that the optimal (directional) estimation error is attained at the $p$ where both terms are approximately the same.

\vspace{2mm}
\section{Proof Idea and Discussion of the Assumptions}
\label{sec:proofidea}
We now provide the proof sketch followed by a discussion of the assumptions in Theorem~\ref{thm:regressionlpmain} and \ref{thm:classificationlp}. 
\subsection{Proof idea}
\label{subsec:proofintuition}

The proofs of Theorem~\ref{thm:regressionlpmain} and \ref{thm:classificationlp} follow a localized uniform convergence argument which is standard in the literature on empirical risk minimization. More specifically, we first upper bound the $\ell_p$-norm of the interpolators by 
\begin{equation}
     \min_{ \forall i:\:\langle \xui , w\rangle =y_i } \norm{w}_p^p  \leq \loneboundreg~~\mathrm{and}  \min_{\forall i:\:\yui \langle x, w \rangle \geq 1} \norm{w}_p^p   \leq \loneboundclas.
\end{equation}
In a second step, we uniformly bound the risk over all interpolating models with $\loneboundreg$/$\loneboundclas$-bounded $\ell_p$-norm
\ifarxiv
\begin{align}
    \RiskR(\what) &\leq  \max_{\substack{\norm{w}_p^p \leq \loneboundreg \\ \forall i:\:\langle \xui , w\rangle =\yui}} \norm{w - \wgt}_2^2 =: \Phi_{\mathcal{R}} 
    ~~\mathrm{and}~~
    \RiskC(\what)\leq 2 - 2 \underbrace{\left[
    \min_{\substack{\norm{w}_p^p \leq \loneboundclas\\ 
    \forall i:\:\yui \langle x_i, w \rangle \geq 1}}
    \frac{\langle w, \wgt\rangle}{\norm{w}_2} \right]}_{=: \Phi_{\mathcal{C}}}. 
\end{align}
\else
\begin{align}
    \RiskR(\what) &\leq  \max_{\substack{\norm{w}_p^p \leq \loneboundreg \\ \forall i:~\langle \xui , w\rangle =\yui}} \norm{w - \wgt}_2^2 =: \Phi_{\mathcal{R}} \label{eq:proofideareg} 
        \\
    \RiskC(\what) &\leq 2 - 2 \underbrace{\left[
    \min_{\substack{\norm{w}_p^p \leq \loneboundclas~~\mathrm{and} \\ 
    \forall i:~ \yui \langle x_i, w \rangle \geq 1}}
    \frac{\langle w, \wgt\rangle}{\norm{w}_2} \right]}_{=: \Phi_{\mathcal{C}}}. \label{eq:proofideaclas}
\end{align}
\fi
We now only discuss the uniform convergence argument as the localization step to upper bound $\loneboundreg$ and $\loneboundclas$ follows from a similar argument. 

The proof exploits the assumption that the ground truth is $1$-sparse (the extension to $s$-sparse ground truths is discussed in Appendix~\ref{rm:p}).  We decompose $w = (\wone, \wtwo)$, with $\wone$ the first entry of $w$ (since $\wgt = (1,0,\cdots,0)$), and abbreviate $\eta = \norm{\wtwo}_2$. Furthermore,  we define $\nu = (w -\wgt)_{[1]} = \wone -1$  for regression and $\nu = \wone$  for classification. Define \vspace{-5mm}

\ifarxiv
\begin{align}
         &\phi_{\mathcal{R}} := \max_{(\nu,\eta) \in \Omegareg} \nu^2 +\eta^2
         \leq \max_{\nu \in \Omegareg}~\nu^2 + \max_{\eta \in \Omegareg}~\eta^2  \\\nonumber 
         \mathrm{and}~~
 &\phi_{\mathcal{C}} :=\min_{ (\nu,\eta) \in \Omegaclas} \frac{\nu }{\sqrt{\nu^2+\eta^2}} \geq \left(1 + \frac{\max_{\eta \in \Omegaclas} ~\eta^2}{\min_{\nu \in \Omegaclas}~ \nu^2}  \right)^{-1/2} \label{eq:phimlpv1main} \, ,
    \end{align}
    \else
\begin{align}
         &\phi_{\mathcal{R}} := \max_{(\nu,\eta) \in \Omegareg} \nu^2 +\eta^2
         \leq \max_{\nu \in \Omegareg}~\nu^2 + \max_{\eta \in \Omegareg}~\eta^2  \nonumber \\
 &\phi_{\mathcal{C}} :=\min_{ (\nu,\eta) \in \Omegaclas} \frac{\nu }{\sqrt{\nu^2+\eta^2}} \geq \left(1 + \frac{\max_{\eta \in \Omegaclas} ~\eta^2}{\min_{\nu \in \Omegaclas}~ \nu^2}  \right)^{-1/2} \label{eq:phimlpv1main} \, ,
    \end{align}
\fi
with constraint sets 
\ifarxiv
\begin{multline}
    \Omegareg = 
    \bigg\{(\nu, \eta) \:\vert \:\exists b>0 \subjto (1+\nu)^p + b^p \leq \loneboundreg \text{ and}~
    \frac{1}{n}\norm{H}_q^2 b^2  \geq \frac{1}{n}\sumin \left( \xi_i - Z_i \nu  - \tilde{Z}_i \eta\right)^2 \bigg\} \, ,
\end{multline}
\else
\begin{multline}
    \Omegareg = 
    \bigg\{(\nu, \eta) \:\vert \:\exists b>0 \subjto (1+\nu)^p + b^p \leq \loneboundreg \text{ and} \\
    \frac{1}{n}\norm{H}_q^2 b^2  \geq \frac{1}{n}\sumin \left( \xi_i - Z_i \nu  - \tilde{Z}_i \eta\right)^2 \bigg\} \, ,
\end{multline}
\fi
and
\ifarxiv
\begin{multline}\label{eq:omegaclas}
    \Omegaclas = 
    \bigg\{(\nu, \eta) \:\vert\: \exists b>0 \subjto \nu^p + b^p \leq \loneboundclas \text{ and }
      \frac{1}{n}\norm{H}_q^2 b^2  \geq    \frac{1}{n}\sumin \possq{1-\xi_i \vert Z_i\vert \nu  + \tilde{Z}_i \eta} \bigg\} \, ,
\end{multline}
\else 
\begin{multline}\label{eq:omegaclas}
    \Omegaclas = 
    \bigg\{(\nu, \eta) \:\vert\: \exists b>0 \subjto \nu^p + b^p \leq \loneboundclas \text{ and } \\
      \frac{1}{n}\norm{H}_q^2 b^2  \geq    \frac{1}{n}\sumin \possq{1-\xi_i \vert Z_i\vert \nu  + \tilde{Z}_i \eta} \bigg\} \, ,
\end{multline}
\fi
where ${(\cdot)_+ = \max(0,\cdot)}$ and  $H,Z, \tilde{Z}$ are \iid~ Gaussian random vectors.

The key ingredient for the proofs is now that via the (Convex) Gaussian Minimax Theorem (C)GMT \cite{thrampoulidis2015regularized,gordon1988milman},
we can show that $\prob(\Phi_{\mathcal{R}} > t \vert \xi, X') \leq 2 \prob(\phi_{\mathcal{R}}\geq t \vert \xi, X' )$ and $\prob(\Phi_{\mathcal{C}} < t \vert \xi, X') \leq 2 \prob(\phi_{\mathcal{C}} \leq t\vert \xi, X')$ with $X'$ being the subset of features $X$ in the direction of the ground truth. In short, high probability upper and lower bounds for $\phi_{\mathcal{R}}$ and $\phi_{\mathcal{C}}$, yield corresponding high probability bounds for $\Phi_{\mathcal{R}} $ and $\Phi_{\mathcal{C}}$. The (C)GMT has been used previously to obtain similar bounds for regression in \cite{koehler2021uniform,wang2021tight}.

The result can then be obtained by 
carefully bounding the constraint sets $\Omegaclas$ and $\Omegareg$ 
using tight concentration inequalities for the $\ell_q$-norm of \iid~Gaussian random vectors in \cite{paouris_2017}.

\vspace{4mm}
\subsection{Limitations and discussion of the assumptions}
\label{rm:p}
In the following subsections, we  now discuss the assumptions of our main results and how they can be generalized. 

\label{apx:remarks}
\paragraph{Assumption on  $p~ \&~ d$}
While for very large $d$, Theorems~\ref{thm:regressionlpmain} and \ref{thm:classificationlp} hold for most $p\in(1,2)$, even close to $1$, for fixed $n,d$ our theorems apply to the range $p\in \left(1+\frac{\kappa_2}{\log\log d},2\right)$.
This assumption is used to obtain a high probability concentration of the dual norm $\norm{H}_q$ in Lemma~\ref{lm:dualnormconc}.
It is possible and straightforward to relax the assumption on $p$ to allow for smaller $p$, such as  $p \in (1+\frac{\kappa \log \log d}{\log d},2)$, by choosing $\epsilon(n,d)$ in Lemma~\ref{lm:dualnormconc} larger. However, this choice comes at the price of non-matching upper and lower bounds. In fact, when choosing $p =  1 + \frac{\kappa}{\log d}$, the $\ell_q$-norm $\norm{H}_q$ behaves similarly as the $\ell_{\infty}$-norm $\norm{H}_\infty$ (see \cite{paouris_2017}).

\paragraph{Assumption on the sparsity of $\wgt$}
In this paper we study the special case where the ground truth is  $1$-sparse and thus aligns "maximally" with the sparse inductive bias of the $\ell_1$-norm. 
Precisely, choosing a $1$-sparse ground truth significantly simplifies the analysis 
(see Section~\ref{sec:proofidea}) and the  presentation of the bounds in Theorem~\ref{thm:regressionlpmain} and \ref{thm:classificationlp}. However, the application of the (C)GMT in Propositions~\ref{prop:CGMT_applicationreg} and \ref{prop:CGMT_application_classification} holds more generally for $s$-sparse ground truths and even non-sparse ground truths as exploited in \cite{koehler2021uniform}. 
Therefore, the proof methodology presented in this paper can 
also be employed to bound the risk of general $s$-sparse ground truths. However, this would come at the cost of more involved theorems statements, and non-tight upper and lower bounds in $s$.

\paragraph{Assumption on the noise}
Theorems~\ref{thm:regressionlpmain} and \ref{thm:classificationlp} assume that the amount of noise $\sigma$ is fixed and non-vanishing as $n\to\infty$.
This setting is of particular interest as the ground truth is still consistently learnable (unlike in settings where the noise dominates), while the noise prevents  the min-$\ell_1$-norm/max-$\ell_1$-margin interpolators from generalizing well.  In fact, we are the first to prove, for a constant noise setting, that we can consistently learn with min-norm/max-margin interpolators at fast rates.

On the other hand, several works including \cite{chinot_2021,chinot2021adaboost,wojtaszczyk_2010} have also studied the low noise regime where $\sigma\to 0$  as $n,d$ tend to infinity. We remark that  our proof of Theorem~\ref{thm:regressionlpmain} can be directly extended to cover vanishing noise for regression.  On the other hand, the methodology used for the classification results, i.e. Theorem~\ref{thm:classificationlp}, strongly relies on the noise model and we leave the extension of our results to low noise settings (i.e. vanishing fraction of flipped labels)  as an interesting future work.

\paragraph{Assumption on the distribution of the features}
Our proofs strongly rely on the (C)GMT (Propositions~\ref{prop:CGMT_applicationreg} and \ref{prop:CGMT_application_classification}), which crucially hinges on the assumption that the input features are Gaussian. We believe that generalizing the input distribution $\prob_X$ requires the development of novel tools -- an important task for future work. 
For instance, the small-ball method that is known to yield tight bounds for general input distributions for many estimators  \cite{mendelson_2014, koltchinskii_2015} results in loose bounds when applied directly (\cite{chinot_2021}, see also the discussion on this topic in \cite{wang2021tight}). It remains to be seen whether a modified technique based on the small-ball method can be  powerful enough to yield tight bounds for min-norm/max-margin interpolators studied in this paper.

\vspace{1.5mm}
\section{Related Work}
The majority of works that attempt to rigorously understand minimum-norm/max-margin interpolation from a non-asymptotic viewpoint have so far
focused on $\ell_2$-norm interpolators, \cite{bartlett_2020,tsigler_2020,muthukumar_2020}.
However, to be asymptotically consistent as $d,n\to\infty$,  the covariates need to be effectively low-dimensional and aligned with the direction of the ground truth, e.g.~via a spiked covariance structure. 

On the other hand, the existing literature on structured interpolators primarily focuses on the min-$\ell_1$-norm interpolator \cite{wang2021tight,muthukumar_2020,chinot_2021,wojtaszczyk_2010,chatterji_2021}, showing exact rates of order $\frac{1}{\log(d/n)}$. Even though  \cite{koehler2021uniform,zhou2021optimistic,chinot_2021} present frameworks to obtain non-asymptotic bounds for general min-norm interpolators, the discussion of the implications of their results center around the  
$\ell_1$/$\ell_2$-norms.
For classification, so far the only known non-asymptotic upper bound holds for the max-$\ell_1$-margin interpolator --- however, assuming that a constant fraction of data points are mislabeled, the upper bound diverges with growing sample size  \cite{chinot2021adaboost}. 

Beyond the mentioned non-asymptotic results, many more papers study linear regression and classification in the limit as $d,n \to \infty$ \cite{muthukumar2021classification, hastie2019surprises, dobriban_2018,deng21,li2021minimum}.
In contrast to this paper, these works study the linear regime where ${d/n \to \gamma>0}$ and thus where the errors of interpolators do not vanish.


\vspace{1.5mm}
\section{Generality and Future Work}
\label{sec:trade-off}
Our results naturally suggests a variety of impactful follow-up work that may advance the understanding of the effect of inductive bias and interpolation both from a theoretical and empirical perspective.

\subsection{Future  Work on Theory for Linear Interpolators}
\label{subsec:optimalp}
As a specific question for future work, a natural quantity of interest would be the optimal choice  among all $p \in (1,2)$ as a function of  $n,d$. 
Unfortunately, our presented bounds are not sufficient to provide an explicit answer.
Note that for $\beta <2$ and $p = 1 +\frac{\kappa_2}{ \log \log d}$, our rates match the uniform lower bounds for all interpolators when $n,d$ are sufficiently large. 
On the other hand, for larger $\beta$, the optimum of the upper bound is achieved at $p< 1 + \frac{\kappa_2}{\log \log d}$
which is beyond the range of our analysis. We refer to Section~\ref{rm:p} for a discussion on smaller choices of $p$. 
Giving a precise expression and tight lower bounds supporting the choice remains a challenging task.

\begin{figure}[!t]
    \centering
    \ifarxiv
    \includegraphics[width=.5\linewidth]{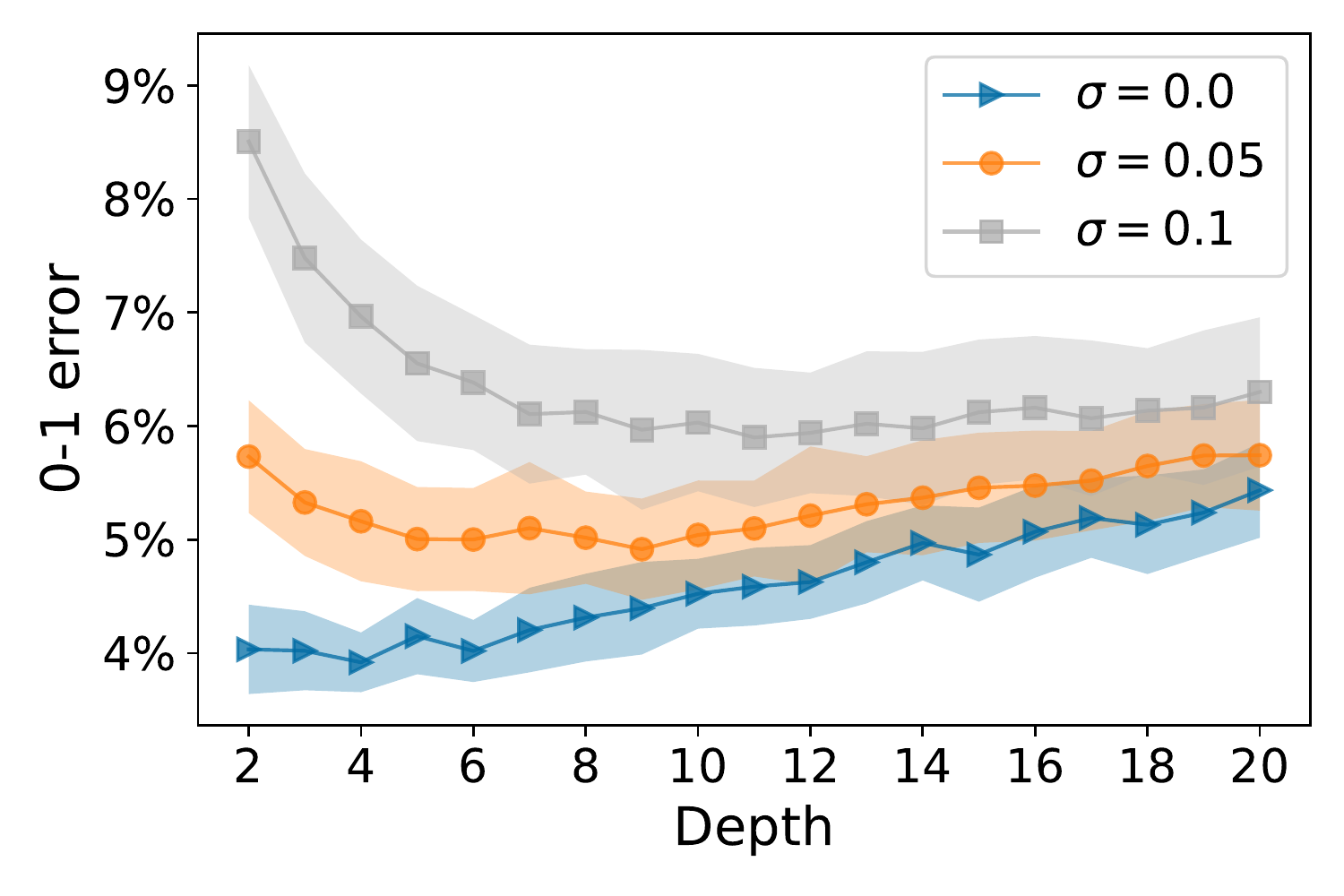}
    \else
    \includegraphics[width=\linewidth]{figures/cntk_depth_vs_risk.pdf}
    \fi
    \caption{Classification with convolutional neural tangent kernels on binarized MNIST. We plot the $0-1$~classification error on the test set as a function of the depth of the model, i.e. the number of stacked convolutional layers with ReLU activation. Means and confidence bands are computed across 50 independent draws of the training set with $n =500$. \vspace{-4mm}}
    \label{fig:cntk}
\end{figure}

\subsection{Non-linear Overparameterized Interpolators}

The discussion in Section~\ref{sec:discussion} suggests that
good generalization of interpolating models in the noisy case hinges on a careful choice of the inductive bias and that stronger is not automatically better, as is  perhaps widely assumed.
This stands in contrast to regularized models and noiseless interpolation, where  the best performance is usually attained at the strongest inductive bias with the right amount of regularization. 
One analogous conclusion is that the optimal inductive bias for \emph{noiseless} interpolation may not be optimal for \emph{noisy} interpolation, in particular for small sample sizes. 
 
As a step towards more complex non-linear models, we provide experimental evidence to support our claim on convolutional neural tangent kernels \cite{cntk}. Kernel regression with kernels with convolutional filter structure has shown good performance on real world image datasets  \cite{lee2020} and is thus a good candidate to move towards state-of-the-art models. 
In Figure~\ref{fig:cntk}, we plot the test error as a function of the depth in a binary classification task on the MNIST data for a  small sample size regime (see Appendix~\ref{appendixsec:cntk} for the experimental details).  
When we do not add artificial label noise in the training data, we observe how increasing the depth leads to a monotonic increase in error, suggesting that the inductive bias decreases with depth. However, when adding label noise, the optimal performance is attained at a medium depth and hence a moderate  inductive bias.
These findings suggest a promising avenue for empirical and theoretical investigations evolving around the new bias variance trade-off.

\section{Conclusion}
\label{sec:conc}
In this paper, we showed that min-$\ell_p$-norm/ max-$\ell_p$-margin interpolators can achieve much faster rates with a moderate inductive bias, i.e.,  $p\in(1,2)$,  compared to a strong inductive bias, i.e., $p=1$. 
This arises from a novel bias-variance type trade-off induced by the inductive bias of the interpolating model, balancing the regularizing effect of high-dimensionality and the structural alignment with the ground truth. Based on preliminary experiments on image data with the CNTK in Figure~\ref{fig:cntk}, we further hypothesize that this trade-off carries over to more complex interpolating models and datasets used in practice and leave a thorough investigation as future work. 

\section*{Acknowledgements}
K.D.~is supported by the ETH AI Center and the ETH Foundations of Data Science. N.R.~is supported by the Max Planck ETH Center for Learning Systems.
We would like to thank Nikita Zhivotovskiy for helpful discussions and Kai Lion for his preliminary experiments. Finally we would like to thank the anonymous reviewers for their feedback.



\bibliography{references}
\bibliographystyle{icml2022}

\newpage
\appendix
\onecolumn

\section{Experimental details for the convolutional NTK}
\label{appendixsec:cntk}
We include here details on the experiments with the convolutional neural tangent kernel (CNTK) presented in Section~\ref{sec:trade-off}. 

\paragraph{Binarized MNIST} We perform a binary classification task on a reduced version of the MNIST data set \cite{deng2012mnist}. In particular, we only utilize the digits from 0 to 5 (included), and define the binary labels as $y=-1$ if the digit is even, $y=+1$ if it is odd. To stay closer to the high dimensional regime, we randomly subsample only $n=500$ data points for training, the remaining ones are used for testing. The label noise is added by randomly flipping exactly $\sigma \%$ of the labels only on the training set.

\paragraph{The CNTK} The CNTK is a kernel method based on the analytical solutions of infinite width neural networks trained via gradient flow. In particular, a convolutional neural network with a given architecture has a corresponding CNTK, defined as its limit for an infinite number of channels and trained to convergence via MSE loss. We include here a short description of the model. For additional details we refer to \cite{cntk, neuraltangents2020} and references therein. \\
Consider a given neural network architecture $f_\theta$, parameterized by $\theta$. \cite{jacot2018neural, cntk} showed that, under initialization and optimization conditions, the CNTK can be formalized as a kernel method with
\begin{equation}
    k(x, x') = \mathbb{E}_\theta \left< 
        \frac{\partial f_\theta(x)}{\partial \theta} ,
        \frac{\partial f_\theta(x')}{\partial \theta}
    \right> \, .
\end{equation}
For Gaussian weights, the expected value above can be computed in closed form via recursion through the layers, which allows this model to be used in practice and without sampling.

\paragraph{Architecture} All the implementations of the CNTK are done via the Neural Tangents Python library \cite{neuraltangents2020}. In our experiments, we utilize CNTK architectures with the following structure: 
\begin{itemize}
    \item we stack convolutional layers followed by a ReLU activation $l$ times, where $l$ is what we call the depth of the architecture.
    We utilize a kernel size 3, stride 1 and padding at every layer. 
    \item a final flatten layer followed by a linear layer with one output neuron, containing the logit of the classification probability.
\end{itemize}

The results for this experiment are presented in Figure~\ref{fig:cntk}.

\section{Comparison of Theorem~\ref{thm:regressionlpmain} and \ref{cr:koehler}}
\label{rm:koehler}
We now discuss how our proof of Theorem~\ref{thm:regressionlpmain} differs from the proof of Theorem~\ref{cr:koehler} (i.e., Theorem~4 \cite{koehler2021uniform}). Theorem 4 in \cite{koehler2021uniform} (summarized in  Theorem~\ref{cr:koehler}) also applies to non-sparse ground truths $\wgt$ and is therefore expected to be less tight. Nevertheless, we  discuss the major differences with respect to their proof, which allow us to obtain tighter bounds and illustrate the different rates in Figure~\ref{fig:ratesk}.
Balancing the terms $\sqrt{\frac{d^{2/q}}{n}} $ and $\frac{1}{d^{1/q}}$, we obtain an optimal rate of order~$n^{-1/4}$ in Theorem~\ref{cr:koehler}. We now discuss how each of these terms can be further tightened in the bounds in Theorem~\ref{thm:regressionlpmain}.

\paragraph{Use of triangular inequality in $\PhiNR,\PhipR$ and $\PhimR$.} The definitions of the terms  $\PhiNR,\PhipR$ and $\PhimR$  in the proof of Theorem~\ref{thm:regressionlpmain} in Appendix~\ref{apx:regproof} differ from the ones used in \cite{koehler2021uniform} for the localization and uniform convergence steps. In \cite{koehler2021uniform}  the authors rely on  the triangular inequality in the localization and  uniform convergence steps by bounding 
\begin{align}
   &\PhiNR := \norm{\what}_p^p = \min_{X w = \xi}  \norm{w+\wgt}_p^p
        \leq \|\wgt\|_p +  \min_{X w = \xi}  \norm{w}_p \leq \widetilde{\lonebound} \, ~~~\mathrm{and}\\
    & \PhipR :=
\max_{\substack{
      \norm{w+\wgt}_p \leq  \widetilde{\lonebound} \\ 
            Xw=\xi
        }}  \norm{w}_2^2 \leq \max_{\substack{
            \norm{w}_p \leq  \widetilde{\lonebound} + \norm{\wgt}_p \\ 
            Xw=\xi
        }}  \norm{w}_2^2.
\end{align}
While such a procedure also covers non-sparse ground truth, the resulting bounds are not tight when $\wgt$ is sparse.
Instead, we directly bound the terms $\PhiNR$, $\PhipR$(for which we make use of the sparsity of $\wgt$), which effectively allows us to  improve the terms of order $\sqrt{\frac{  d^{2/q}}{n}}$ in Theorem~\ref{cr:koehler} to the term of order $\frac{\sigma^{4-2p}q^pd^{2p-2}}{ n^p}$ in Theorem~\ref{thm:regressionlpmain}.

\paragraph{Bounds for $\fnreg$.} Another important improvement is in the bounds used for $\fnreg$. While the proofs in \cite{koehler2021uniform} rely on a bound of the form of Equation~\eqref{eq:lpreghighnoisestep1}, in order to obtain faster rates, our analysis requires on the tighter bound in Equation~\eqref{eq:regressionfn1}. 
More specifically, in Theorem~\ref{thm:regressionlpmain} this allows us to avoid the term  $\sqrt{\frac{\log d}{n}}$.

\paragraph{Concentration inequalities for $\norm{H}_q$.}
The concentration inequalities for the dual norm $\norm{H}_q$ are a crucial ingredient for the proof of the theorem. The ones used in the proofs in \cite{koehler2021uniform} rely on the concentration of Lipschitz continuous functions (see, \cite{Ledoux1992AHS}) and also apply to more general norms. However, for the specific case of $\ell_p$-norms we can make use of much tighter concentration inequalities from \cite{paouris_2017} (see Appendix~\ref{apx:proofdualnormconc}) to avoid the terms involving $\sqrt{\frac{\log(d)}{q}} \frac{1}{ d^{1/q}}$  in the bound in Theorem~\ref{cr:koehler}.

\begin{figure}[t]
    \centering
    \includegraphics[width=0.5\linewidth]{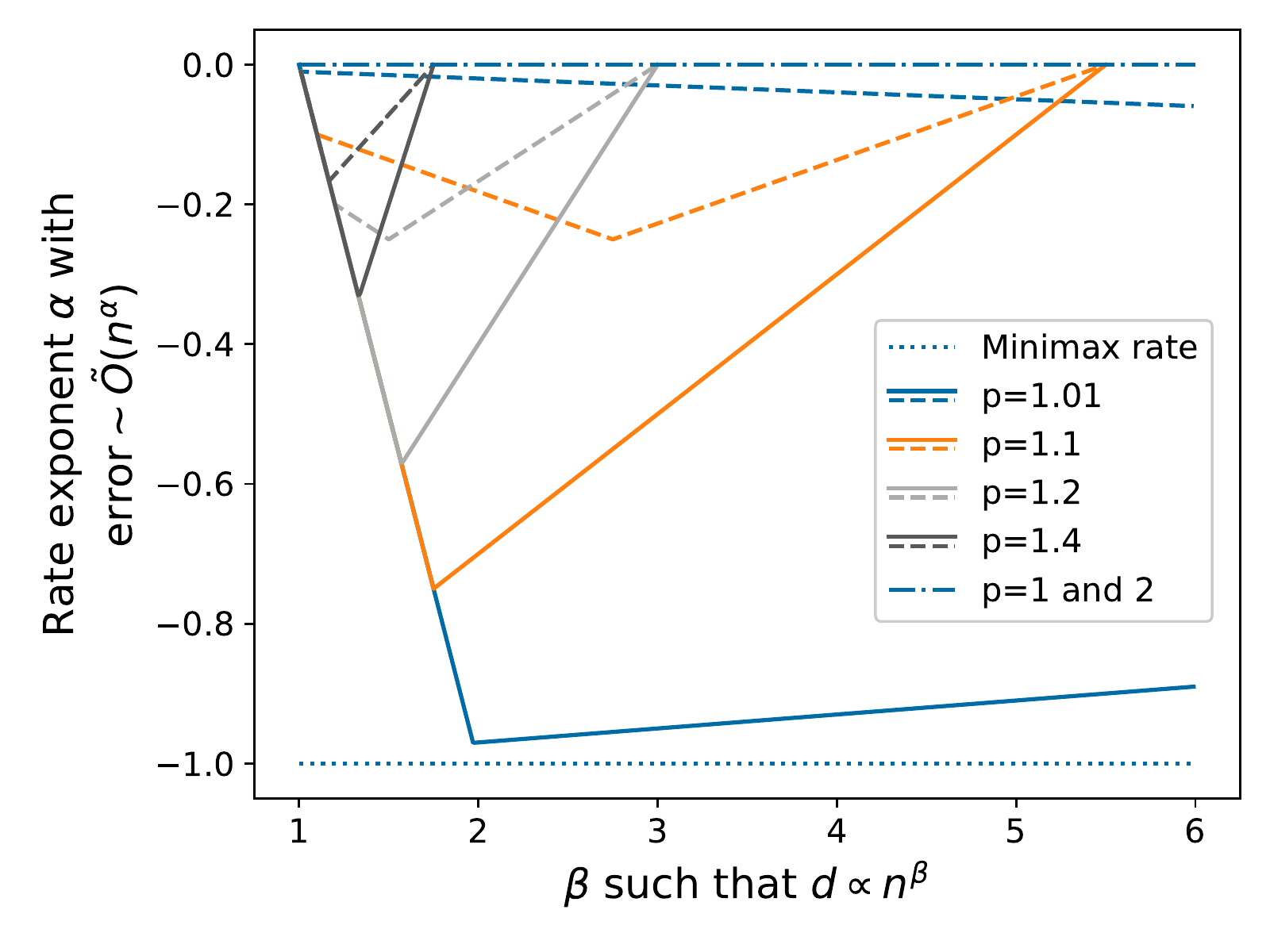}
         \caption{Depiction of the rates in  Theorem~\ref{thm:regressionlpmain} (solid line) and Theorem~\ref{cr:koehler} (dashed lines): classification. Order of the theoretical rates $\tilde{O}(n^{\alpha})$ for different strengths of inductive bias $p$ at different asymptotic regimes. On the x-axis we represent the value $\beta$ defining the regime $d=n^{\beta}$, on the y-axis the risk decay $\alpha$.}
         \vspace{-0.1in}
    \label{fig:ratesk}
\end{figure}

\section{Proof of Theorem~\ref{thm:regressionlpmain}}
\label{apx:regproof}
Throughout the remainder of the appendix we use $\kappa_1,\kappa_2,...$ and $c_1,c_2,...$ for generic universal positive constants independent of $d$, $n$ or $p$. The value $c_1,c_2,...$ may change from display to display throughout the derivations. The standard notations $O(\cdot), o(\cdot), \Omega(\cdot), w(\cdot)$ and $\Theta(\cdot)$, as well as $\lesssim,\gtrsim$ and $\asymp$, are utilized to hide universal constants, without any hidden dependence on $d$, $n$ or $p$.   
Throughout the proof, whenever we say with high-probability, we mean with probability $\geq 1 - c_1 d^{-c_2}$ with universal constants $c_1,c_2 >0$.

The lower bound $\RiskR(\what) \gtrsim \frac{n}{d}$ follows  simply from the uniform lower bound for all interpolators in Theorem~1 and Corollary~1 in \cite{muthukumar_2020}. 

For the rest of the proof of the statement, 
we use a  localization/maximization approach,
similar to the papers \cite{chinot_2021,koehler2021uniform,ju_2020,muthukumar_2020,wang2021tight}, and common in the literature. More specifically, the proof consists of two major parts:
\begin{enumerate}
    \item \emph{Localization.}
    We derive a  high-probability upper bound $\loneboundreg$ in Proposition~\ref{prop:phidnhighreg} that only depends on the noise $\xi$ but not on the draws of the features $X$. 
    More specifically, for the norm of the min-$\ell_p$-norm interpolator $\what$
    \begin{align} \label{eq:Phi_Nreg} 
         \PhiNR := \norm{\what}_p^p =  \min_{ X(w-\wgt) = \xi} & \norm{w}_p^p
        = \min_{X w = \xi}  \norm{w+\wgt}_p^p
        \leq \loneboundreg
    \end{align}
    and consequently
    \mbox{$\norm{\what}_p^p \leq \loneboundreg$},
    with high probability.
    
    \item \emph{Uniform convergence.}
     Conditioning on the  draw of the noise $\xi$, we derive high-probability uniform upper and lower bounds (Proposition~\ref{prop:highnoiselpreg}) on the prediction error for all interpolators $w$ with  $\|w\|_p^p \leq \loneboundreg $, 
    \begin{align} \label{eq:Phi_+reg}
        \PhipR := \max_{\substack{
            \norm{w}_p^p \leq \loneboundreg \\ 
             X(w-\wgt)=\xi
        }} & \norm{w-\wgt}_2^2 
        = \max_{\substack{
            \norm{w+\wgt}_p^p \leq \loneboundreg \\ 
            Xw=\xi
        }}  \norm{w}_2^2
        ~~~\mathrm{and}\\
     \label{eq:Phi_-reg} 
            \PhimR := \min_{\substack{
            \norm{w}_p^p \leq \loneboundreg \\ 
             X(w-\wgt)=\xi
        }} & \norm{w-\wgt}_2^2 
        = \min_{\substack{
            \norm{w+\wgt}_p^p \leq \loneboundreg \\ 
            Xw=\xi
        }}  \norm{w}_2^2.
    \end{align}

\end{enumerate}

\subsection{Application of the (C)GMT}
\label{apx:subsec:gmtreg}
Similar to \cite{koehler2021uniform}, we now discuss how the (Convex) Guassian Minimax Theorem ((C)GMT) \cite{gordon1988milman,thrampoulidis2015regularized} can be used to simplify the study of the primal optimization problems defining $\PhiNR,\PhipR$ and
$\PhimR$ via the introduction of auxiliary optimization problems $\phiRdn,\phiRdp$ and $\phiRdm$ defined below. This approach also applies to general sparse ground truths with $\norm{\wgt}_0 =s$. 

We begin with the introduction of some additional notation. Assume that $\wgt$ is $s$-sparse which motivates us to separate $w$ into the vector $\wone \in \RR^s$ consisting of the $s$-sparse entries containing the support of $\wgt$ and $\wtwo \in \RR^{d-s}$ consisting of the remaining entries. Similarly, decompose $\wgt$ into $\wgtp$ and $\wgtp'=0$. 
Using this notation, we can express $\yui \langle \xui, w \rangle = \yui \left(\langle \xsui, \wone \rangle + \langle \xscui, \wtwo \rangle\right)$ with $\xsui$ and $\xscui$ the corresponding subvectors of $\xui$.  
We denote by $\Xs = ( x'_1,\cdots,x'_n )$ and $X'' = (  x''_1,\cdots,x''_n )$ the sub-matrices  of $X=(x_1,\cdots,x_n )$.

 We now make use of the fact that $\xscui$ is independent of $\yui$ and $\xsui$ and define $\fnreg : \RR^{s}\times\RR \to \RR_+$ with 
    \begin{equation}
        \fnreg(\wone,\|\wtwo\|_2 ) = \frac{1}{n}\| \xi - \Xs \wone - \Ggaussian \|\wtwo\|_2 \|_2^2.
    \end{equation}
  with $\Ggaussian \sim \NNN(0,I_n)$, 
 By reformulating the optimization problems via  Lagrange multipliers, we can then apply the CGMT to  $\PhiNR$ and the GMT to  $\PhipR,\PhimR$. As a result, we obtain the following result: 

\begin{proposition}
\label{prop:CGMT_applicationreg}
      Let $p \geq 1$ be some constant and assume that $\wgt$ is $s$-sparse. 
  For $\Hgaussian \sim \NNN(0,I_{d-s})$ define the stochastic auxiliary optimization problems:\footnote{We define $\PhiNR,\PhimR,\phiRdn,\phiRdm = \infty$ and $\PhipR,\phiRdm = -\infty$ if the corresponding optimization problems have no feasible solution.}
        \begin{align}
        \label{eq:phin}
            \phiRdn &= \min_{(\wone,\wtwo)} \norm{\wone + \wgtp}_p^p+\norm{\wtwo}_p^p
            \subjto \frac{1}{n} \innerprod{\wtwo}{\Hgaussian}^2 \geq  \fnreg(\wone,\|\wtwo\|_2) \\
       \label{eq:phip}
            \phiRdp &= \max_{(\wone,\wtwo)} \norm{\wone}_2^2+\norm{\wtwo}_2^2
            \subjto \begin{cases}
                \frac{1}{n} \innerprod{\wtwo}{\Hgaussian}^2 \geq  \fnreg(\wone,\|\wtwo\|_2)\\
              \norm{\wone+\wgtp}_p^p+\norm{\wtwo}_p^p \leq \loneboundreg
            \end{cases}\\
        \label{eq:phim}
            \phiRdm &= \min_{(\wone,\wtwo)} \norm{\wone}_2^2+\norm{\wtwo}_2^2
            \subjto \begin{cases}
               \frac{1}{n} \innerprod{\wtwo}{\Hgaussian}^2 \geq  \fnreg(\wone,\|\wtwo\|_2) \\
               \norm{\wone + \wgtp}_p^p+\norm{\wtwo}_p^p  \leq \loneboundreg
            \end{cases}
        \end{align}
where $\loneboundreg>0$ is a constant possibly depending on $\xi$ and $\Xs$.  Then for any $t \in \RR$, we have:
        \begin{align}
            \PP( \PhiNR > t \vert \xi, \Xs) &\leq 2\PP( \phiRdn \geq t \vert \xi, \Xs)
            \\
            \PP( \PhipR > t \vert \xi, \Xs) &\leq 2\PP( \phiRdp \geq t \vert \xi, \Xs) 
            \\
            \PP( \PhimR < t \vert \xi, \Xs) &\leq 2\PP( \phiRdm \leq t \vert \xi, \Xs) 
           ,
        \end{align}
        where the probabilities on the LHS and RHS are over the draws of $X''$ and of $\Ggaussian,\Hgaussian$, respectively. 
\end{proposition}   
The proof is analogous to Lemma 4 and Lemma 7 in \cite{koehler2021uniform}. Proposition~\ref{prop:CGMT_applicationreg} allows us to reduce the optimization problems depending on the random matrix $\datamat\in \mathbb{R}^{n \times d}$ to optimization problems depending on the much smaller random matrix $\Xs \in \mathbb{R}^{n \times s} $ and two additional random vectors $\Hgaussian,\Ggaussian$.

\subsection{Localization step - bounding \texorpdfstring{$\phiRdn$}{phiN}}
\label{apx:subsec:localisationproofreg}

The goal of this first step is to give a high probability upper bound on $\phiRdn$ from Proposition~\ref{prop:CGMT_applicationreg}. We now return to the $1$-sparse ground truth $\wgt=(1,0,\cdots,0)$ assumed in the theorem statement, which allows us to write $\phiRdn$ as
\begin{equation}
\label{eq:lpregphidnproof}
    \phiRdn = \min_{(\wone,\wtwo)} \vert \wone + 1\vert^p+\norm{\wtwo}_p^p
            \subjto \frac{1}{n} \innerprod{\wtwo}{\Hgaussian}^2 \geq  \fnreg(\wone,\|\wtwo\|_2)
\end{equation}

\begin{remark}[A simplifying modification for the proof.]
\label{rm:simplicification}
By definition,  $\Hgaussian \in \RR^{d-1}$. Instead, for the ease of presentation, we  prove the result assuming that $\wgt, x_i \in \RR^{d+1}$ and hence $\Hgaussian \in \RR^d$, which does not affect the expression for the bounds in Equation~\eqref{eq:thmregboundhighnoiseuppermain} in Theorem~\ref{thm:regressionlpmain}. 
\end{remark}

We now  derive  a bound for $\phiRdn$ that holds with high probability over the draws of $H,G,X',\xi$ and define $\dmoment{q} = \EE \|H\|_{q}$ and  $\sigmaxi^2 = \frac{1}{n} \|\xi\|_2^2$ with $H \sim N(0,I_d)$ and $\mutild := \frac{\dmoment{2q/p}^{2q/p}}{\dmoment{q}^{2q/p}}$. 

\begin{proposition}
\label{prop:phidnhighreg}
Let the data distribution be as described in Section~\ref{subsec:regsetting} and assume that $\sigma = \Theta(1)$. 
Under the same conditions as in Theorem~\ref{thm:regressionlpmain} for $n,d,p$, there exist universal constants $c_1,c_2,\cdots,c_8>0$ such that with  probability at least $ 1- c_1 d^{-c_2}$ over the draws of $\Ggaussian,\Hgaussian,\Xs,\xi$, it holds that
\begin{align}
\phiRdn  \leq \left(\frac{n \sigmaxi^2}{\dmoment{q}^2}\right)^{p/2}\left(1+ \frac{p}{2} \frac{\nu_0^2}{\sigmaxi^2}  \left(1+ c_6 \rho \right) + \OOOc + c_7 \left( \rho^2 + \frac{\rho \vert \nu_0 \vert}{\sigmaxi}\right) \right)
+ \nugt + p\nu_0\left(1+ c_8 \nu_0 \right) =: \loneboundreg,
\label{eq:lpphinbound}
\end{align}
with $\OOOc = c_5 \frac{ n \exp(c_3q)}{d q }$,  $\rho = \frac{\log^{c_4}d}{\sqrt{n}}$ and $\nunull =  -\sigmaxi^{2} \left(\frac{\dmoment{q}^2}{n \sigmaxi^2}\right)^{p/2} =  \Theta \left( \frac{q^p d^{2p-2}}{n^p} \right)^{1/2}$.
\end{proposition}
Note that $\loneboundreg$ only depends on $\xi$  but not on $\Hgaussian$ and $\Ggaussian$, and thus, $\loneboundreg$ is a valid choice for the upper bound of the $\ell_p$-norm of $w$ in the constraints of $\PhipR,\PhimR$ and  Proposition~\ref{prop:CGMT_applicationreg}.  Furthermore, we note that the bound in the proposition holds with high-probability over the draws of $\xi$ despite explicitly depending on $\sigmaxi$.  However, replacing $\sigmaxi$ with $\sigma$ in the Equation~\eqref{eq:lpphinbound} by applying the high-probability bound $\vert \sigmaxi^2 -\sigma^2\vert \lesssim \sqrt{\frac{\logjone{d}}{n}}$ \fy{fix this} would lead to a loose bound in the subsequent uniform convergence step (Section~\ref{apx:subsec:unfiformconvreg}).
\begin{proof}

We condition on the high probability event $\vert \sigmaxi - \sigma\vert \lesssim \sqrt{\frac{\log d }{n}}$ and $\sigma \asymp 1$ by assumption, which we will use in multiple occasions implicitly throughout the reminder of the proof. 
By definition, any feasible $w=(w',w'')$ in Equation~\eqref{eq:lpregphidnproof} must satisfy $\frac{1}{n} 
\innerprod{\wtwo}{\Hgaussian}^2  \geq \fnreg(\wone,\|\wtwo\|_2 ) $.
Since the goal is to upper bound the solution of the minimization problem in Equation~\eqref{eq:lpregphidnproof}, it suffices to find one feasible point. 
In a first step, 
following standard concentration arguments (as used for instance in the  proof of Lemma 5 in \cite{koehler2021uniform}), we can show that with probability at least $ 1-c_1 d^{-c_2}$ the following equation holds uniformly over all $w',w''$ 
\begin{align}
    \label{eq:regressionfn1}
 \vert \fnreg(\wone, \norm{\wtwo}_2) - \sigmaxi^2 -(\wone)^2 -\norm{\wtwo}_2^2\vert &\lesssim \tilde{\rho}\left[\sigmaxi  (\norm{\wtwo}_2 + \vert \wone \vert) + (\wone)^2 +\norm{\wtwo}_2^2\right]
\end{align}
with $\tilde{\rho} = \sqrt{\frac{\logjone{d}}{n}} $. 
We now show that we can choose $\wone = \nu_0$  and   $\wtwo$ to be a rescaled version of the subgradient of $H$, i.e. $\wtwo = b \wgrad = b \partial \|H\|_q$ with some $b \in \RR_+$. We have $\langle \wtwo, H \rangle = \|w\|_p \norm{H}_q =  b \norm{\wgrad}_p \norm{H}_q  = b \norm{H}_q$ and  thus suffices to show that
\begin{align}
\frac{b^2 \norm{H}_q^2}{n} &\geq  \sigmaxi^2+  (\nunull^2+b^2 \|\wgrad\|_2^2)(1+ c_1 \tilde{\rho} ) + c_2 \tilde{\rho} \sigmaxi (\vert \nunull\vert+b\|\wgrad\|_2) \\ 
\impliedby
b^2 &\geq \frac{ \sigmaxi^2 +  \nunull^2(1+ c_1 \tilde{\rho} ) + c_2 \tilde{\rho} \sigmaxi (\vert \nunull\vert+b\|\wgrad\|_2) }{\frac{\norm{H}_q^2}{n} - \norm{\wgrad}_2^2 (1+ c_1 \tilde{\rho} ) }. \label{eq:lpboundphinb} 
\end{align}

\paragraph{Finding feasible $b$} We begin by studying the denominator, for which we can use concentration results for the $\ell_q$-norm of Gaussian random vectors in \cite{paouris_2017}, summarized in the following lemma
\begin{lemma}
\label{lm:dualnormconc}
There exist universal constants $c_1,c_2,\cdots,c_6>0$ such that for any $n \geq c_1$, $n \lesssim d \lesssim \exp(n^{c_2})$, $q \lesssim \log\logjone{d}$ 
and for $\epsilon = \max\left(\frac{n}{d}, \frac{\log^{c_3}(d)}{n}\right)$ 
 with probability at least $ 1-c_4 d^{-c_5}$, 
\begin{align}
\langle \wgrad, H\rangle & = \dmoment{q}\left(1+O\left(\epsilon \right)\right)   
~~\mathrm{and} \label{eq:dualnormconc1}
\\
\norm{\wgrad}_2^2 &= 
\mutild
\left( 1 + O\left(\epsilon q \right)\right)  \label{eq:dualnormconc2}
\end{align}
where $p$ is such that $\frac{1}{p} + \frac{1}{q} = 1$. 
Further it holds that
\begin{equation}
    \dmoment{q} = \Theta(d^{\frac{1}{q} + \frac{\log q}{2 \log d}}) \:\: \text{and} \:\: \mutild = O\left(d^{\frac{1}{q}-\frac{1}{p} + \frac{c_6 q}{\logjone{d}}}\right) \lor \Omega\left(d^{\frac{1}{q}-\frac{1}{p} - \frac{c_6 q}{\logjone{d}}}\right). 
\end{equation}
\end{lemma}

In particular, as a consequence of Lemma~\ref{lm:dualnormconc} and Proposition~\ref{cor:hqbound2}, we have with high probability that

\begin{align}
    \frac{\norm{H}_q^2}{n} - \norm{\wgrad}_2^2(1+ c_1 \tilde{\rho}) &= \frac{\dmoment{q}^2}{n}(1+O(\epsilon)) - \mutild (1+O(q\epsilon))(1+ c_1 \tilde{\rho} ) \\
    &=\frac{\dmoment{q}^2}{n} \left(1 + O(\epsilon) - \frac{\mutild n}{ \dmoment{q}^2} (1+O(q \epsilon +
        \tilde{\rho}))\right),
    \label{eq:regapplyconc}
\end{align}
with $\epsilon =  \max\left(\frac{n}{d}, \frac{\log^{2 c}d}{n}\right) = O(\frac{ n \exp(c_3q)}{d q } +\rho^2)$ and 
$  \rho = \frac{\log^{c} d}{\sqrt{n}} \geq \tilde{\rho}$ with universal constant $c>1$. Note that this step requires the assumption that $q \lesssim \log\log d$ or equivalently $p \geq 1 + \frac{\kappa_2}{\log\log d}$ when choosing $\epsilon$. 
In fact, the small choice of $\epsilon$ is necessary such that the error term $1 + O(\epsilon)$ arising from the concentration in  Equation~\eqref{eq:regapplyconc} does not affect the final error bound in Theorem~\ref{thm:regressionlpmain}.
 
Next, it is straight forward to check that both terms $\nunull, \frac{\mutild n}{ \dmoment{q}^2} \to0$ using Proposition~\ref{cor:hqbound2} and the assumptions in the theorem statement on $n,d$ and $p$. More precisely, we choose $\kappa_1,\cdots, \kappa_4>0$ in the theorem statement such that 
\begin{align}
 &\frac{\mutild  n}{\dmoment{q}^2} = O\left(\frac{ n \exp(c_3q)}{d q }\right)  = O\left(\frac{1}{\log d}\right)
\end{align}
As a result, Equation~\eqref{eq:regapplyconc} becomes
\begin{align}
        \frac{\norm{H}_q^2}{n} - \norm{\wgrad}_2^2(1+c_1\tilde{\rho}) =
        \frac{ \dmoment{q}^2}{n }\left(1 +  O\left(\frac{ n \exp(c_3q)}{d q } + \rho^2 \right) \right),
\end{align}
where we further used that $O(q\epsilon) = O\left(\frac{\log\log d }{\log d}\right)$.

We now study the the nominator.
Recall that we conditioned on the event where $\sigmaxi$ is lower bounded by a constant. 
We again use Lemma~\ref{lm:dualnormconc}  for the concentration of $\|\wgrad\|_2$ around $\sqrt{\mutild}$ and $\sqrt{\mutild} = 
\frac{\OOOc^{1/2} \dmoment{q}}{n^{1/2}}$, with  $\OOOc = c_5 \frac{ n \exp(c_3q)}{d q }$,  which yields 
that it suffices to satisfy
\begin{align}
b^2 &\geq \frac{ \sigmaxi^2 \left( 1 +  \frac{\nunull^2}{\sigmaxi^2}(1+c_1 \rho) + c_2\frac{\rho}{\sigmaxi}\left(\vert \nunull\vert+b \frac{\OOOc^{1/2} \dmoment{q}}{n^{1/2}}\right)\right)}{        \frac{ \dmoment{q}^2}{n }\left(1 +c_4\left(\frac{ n \exp(c_3q)}{d q } + \rho^2 \right) \right)}.
\end{align}

We can now combine our bounds on the denominator and nominator using Equation~\eqref{eq:lpboundphinb}. First recall that $\tilde{\rho} \leq \rho$.
We can now take the Taylor series approximation of $\frac{1}{1+x} = 1 - x + O(x^2)$ to further tighten the inequality in  Equation~\eqref{eq:lpboundphinb}
so that it suffices to find $b$ satisfying
\begin{align}
\label{eq:regloctemp1}
    b^2 &\geq \frac{n \sigmaxi^2}{\dmoment{q}^2} \left( 1 +  \frac{\nunull^2}{\sigmaxi^2}(1+ c_1 \rho) + c_2 \frac{\rho}{\sigmaxi} \left(\vert \nunull\vert+b \frac{\OOOc^{1/2} \dmoment{q}}{n^{1/2}}\right) \right)(1
    +c_3 \rho^2 + \OOOc ),\\
    \impliedby b^2 &\geq  \frac{n \sigmaxi^2}{\dmoment{q}^2} \left( 1 +  \frac{\nunull^2}{\sigmaxi^2}(1+ c_1\rho ) + c_2 \frac{\rho |\nu_0|}{\sigmaxi}
    +c_3\rho^2 + \OOOc \right)+ \frac{ c_4b n^{1/2}\rho \OOOc^{1/2} }{\dmoment{q}}.
\label{eq:tempeqlp}
\end{align}
Furthermore, we choose $\kappa_1,\cdots, \kappa_4>0$ in the theorem statement such that 
$\nu_0^2 = \frac{\dmoment{q}^{2p} \sigmaxi^{4-2p}} {n^p} = \Theta \left( \frac{q^p d^{2p-2}}{n^p} \right) = O\left(\frac{1}{\log d}\right)$ and therefore, we can  choose $b$ satisfying
\begin{equation}
\label{eq:bsatisfies}
b^2 = \frac{n \sigmaxi^2}{\dmoment{q}^2} \left(1+\frac{ \nu_0^2}{\sigmaxi^2}\left(1+ c_1\rho \right) + c_2(\frac{\rho |\nu_0| }{\sigmaxi} + \rho^2) + \OOOc \right). 
\end{equation}
to satisfy Equation~\eqref{eq:lpboundphinb} and hence the feasibility constraint for \texorpdfstring{$\phiRdn$}{phiN}.

We can now obtain the upper bound in Proposition~\ref{prop:phidnhighreg} noting that $\phiRdn \leq b^p + (\nugt + \nu_0)^p $ and 
taking the Taylor series expansion $(1+x)^p \approx 1 + px + O(x^2)$, using that $\nu_0^2 = O\left(\frac{1}{\log d}\right)$ 
\end{proof}

\subsection{Uniform convergence step}
\label{apx:subsec:unfiformconvreg}
We use the same notation as in Appendix~\ref{apx:subsec:localisationproofreg}. We now prove upper and lower bounds for $\phiRdp,\phiRdm$, respectively, for the choice $\loneboundreg$ from Proposition~\ref{prop:phidnhighreg}.Note that both problems only depend on $\wtwo$ via $\norm{\wtwo}_2$ and $\norm{\wtwo}_p$. 
We are now going to decouple these two quantities.  Define 
\begin{equation}
    \Gamma = \{ (\nu, \eta) \in \mathbb{R} \times \mathbb{R}_+\vert \exists b>0 \subjto \frac{1}{n} \norm{H}_q^2 b^2  \geq  \fnreg(\nu, \eta) ~~\mathrm{and}~~ (\nu+\nugt)^p + b^p \leq \loneboundreg \}
\end{equation} 
Using H\"older's inequality (i.e., $\langle w, \Hgaussian \rangle \leq \|\Hgaussian\|_q \|w\|_p$)  to relax the optimization problems defining  $\phiRdp,\phiRdm$ in Proposition~\ref{prop:CGMT_applicationreg}, we can show that
\begin{align}
        \phiRdp &\leq \left[ \max_{(\nu,\eta) \in \Gamma} \nu^2 +\eta^2
        \right]
   ~~ \mathrm{and}~~
        \phiRdm \geq \left[\min_{(\nu,\eta) \in \Gamma} \nu^2 +\eta^2
                  \right].
    \end{align}
Hence,  Equation~\eqref{eq:thmregboundhighnoiseuppermain} in Theorem~\ref{thm:regressionlpmain} follows directly from the following proposition. 
\begin{proposition}
\label{prop:highnoiselpreg}
Let the data distribution be as described in Section~\ref{subsec:regsetting} and assume that $\sigma = \Theta(1)$. 
Under the same conditions as in Theorem~\ref{thm:regressionlpmain} for $n,d,p$  and for the choice $\loneboundreg$ in Proposition~\ref{prop:phidnhighreg}, there exist universal constants $c_1,c_2,\cdots,c_5>0$ such that with probability at least $ 1- c_1 d^{-c_2}$ over the draws of $\Ggaussian,\Hgaussian,\Xs,\xi$, it holds that
\begin{align}
  \Gamma \subseteq \left\{(\nu,\eta) \vert   ( \nu - \nu_0)^2 \lor \eta^2 \lesssim  \nu_0^{3/2} + \nunull^2 \rho^{1/2}  +  \sigma^2 (\OOOc + c_4 \rho^2 ) +  \sigma\rho \nu_0
\right\},
\end{align}
with $\OOOc = c_5 \frac{ n \exp(c_3 q)}{d q } $,   $\rho = \frac{\log^{c_4}d}{\sqrt{n}}$ and $\nunull =  -\sigmaxi^{2} \left(\frac{\dmoment{q}^2}{n \sigmaxi^2}\right)^{p/2} =  \Theta \left( \frac{q^p d^{2p-2}}{n^p} \right)^{1/2}$.
\end{proposition}

\begin{proof}
First we show that any $\nu,\eta \in \Gamma$ vanishes, i.e. $\nu,\eta \to 0$ as $d,n \to \infty$, which we then use for the second step where we derive a tight bound. 
We condition on the high probability event $\vert \sigmaxi - \sigma\vert \lesssim \sqrt{\frac{\log d }{n}}$ and $\sigma \asymp 1$ by assumption, which we will use in multiple occasions implicitly throughout the reminder of the proof.

\paragraph{Step 1: $ \Gamma \subset \left\{(\nu,\eta) \vert \nu^2 + \eta^2 = O\left( \left(\frac{\dmoment{q}^2}{n \sigmaxi^2}\right)^{p/2} + 
  \rho \right)+\OOOc \right\}$}
First note that we can relax the constraints in $\Gamma$ to $\frac{1}{n} \norm{H}_q^2 \loneboundreg^{2/p}  \geq  \fnreg(\nu, \eta) $. Conditioning on Equation~\eqref{eq:regressionfn1} to control $\fnreg$ we then obtain
\begin{equation}
\label{eq:lpreghighnoisestep1}
    \frac{\loneboundreg^{2/p}\|\Hgaussian\|_q^2}{n} \geq  \fnreg(\nu,\eta) \geq \left(\sigma^2 +   \nu^2 +  \eta^2\right)(1+ c_3 \rho),
\end{equation}
which holds with probability
at least $ 1- c_1 d^{-c_2}$, where we again choose $\rho = \frac{\log^{c_4}d}{\sqrt{n}}$ with universal constant $c_4>1$.

Recalling the definition of $\loneboundreg$ from Proposition~\ref{prop:phidnhighreg}, and using the Taylor series approximation $\frac{1}{1+x} = 1 -x + O(x^2)$, it holds that
\begin{align}
\label{eq:lpMunif}
       \loneboundreg^{2/p} \leq \frac{n \sigmaxi^2}{\dmoment{q}^2}\left(1+ \frac{\nu_0^2}{\sigmaxi^2}  \left(1+ c_6 \rho \right) + \OOOc+ c_7 (\rho^2 + \rho \nu_0) + \frac{2 + 2p \nu_0 + c_8 \nu_0^2 }{p \left(\frac{n \sigmaxi^2}{\dmoment{q}^2}\right)^{p/2}} \right),
\end{align}
and hence by Lemma~\ref{lm:dualnormconc} we have 
\begin{equation}
 \frac{\loneboundreg^{2/p}\|\Hgaussian\|_q^2}{n} \leq \sigmaxi^2 \left(1 + c_3 \left(\frac{\dmoment{q}^2}{n \sigmaxi^2}\right)^{p/2}  +\OOOc\right),
 \end{equation}
  where we only keep track of the dominating terms. In particular,  the following upper bound holds from Equation~\eqref{eq:lpreghighnoisestep1},
 
\begin{equation}
\label{eq:bound2lpeq}
   \nu^2 + \eta^2 =O\left( \left(\frac{\dmoment{q}^2}{n \sigmaxi^2}\right)^{p/2} + 
  \rho \right)+\OOOc = O(\nunull + \rho) + \OOOc =:  \OOO_b= O(\frac{1}{\log d}).
\end{equation}

\paragraph{Step 2: Bound in Proposition~\ref{prop:highnoiselpreg}}
Conditioning on  the event where Equation~\eqref{eq:bound2lpeq} holds, we  now show how we can obtain a tighter bound on $\nu^2$ and $\eta^2$ using a more refined analysis.
As  in the proof of Proposition~\ref{prop:phidnhighreg}, we condition on the event where Equation~\eqref{eq:regressionfn1} holds, which allows us to relax the first constraint in  $\Gamma$ to 
\begin{align}
\label{eq:lptightcond1reg}
     \frac{1}{n} \norm{H}_q^2 b^2 \geq   \fnreg(\nu,\eta) &\geq  \sigmaxi^2 + (\nu^2  + \eta^2)(1+ c_1 \rho ) +  c_2 \sigmaxi \rho(\nu + \eta). 
     \end{align}
     

Unlike the previous step, we do not further simplify the second constraint in $\Gamma$ but instead use that $\nu^2 = \OOO_b \to 0$ which allows us to take the Taylor series $(1+\nu)^p \approx 1 + p \nu + O(\nu^2) = 1 + p \nu (1 + \OOO_b^{1/2})$, and hence
\begin{align}
    b^p &\leq \loneboundreg - (\nu+\nugt)^p \leq \loneboundreg - 1 - p \nu (1 + \OOO_b^{1/2})
    \\ &=\left(\frac{n \sigmaxi^2}{\dmoment{q}^2}\right)^{p/2}\left(1+ \frac{p}{2} \frac{\nu_0^2}{\sigmaxi^2}  \left(1+ c_6 \rho \right) + \OOOc + c_7 (\rho^2 + \frac{\rho \nu_0}{\sigmaxi}) \right) 
 + p(\nu_0 - \nu) \left(1+ \OOO_b^{1/2}\right).
\end{align}

Further, using the Taylor series approximation $(1+x)^{2/p} = 1 + \frac{2}{p} x + O(x^2)$ and following the same reasoning as for Equation~\eqref{eq:lpMunif} (by applying Lemma~\ref{lm:dualnormconc}), we get
\begin{align}
    \frac{1}{n} \norm{H}_q^2 b^2 &\leq  
  \sigmaxi^2
    \left(1+  \frac{\nu_0^2}{\sigmaxi^2}  \left(1+ c_6 \rho \right) + \OOOc + c_7 (\rho^2 + \frac{\rho \nu_0}{\sigmaxi})
 + \frac{2(\nu_0 - \nu)}{\left(\frac{n \sigmaxi^2}{\dmoment{q}^2}\right)^{p/2}} \left(1+ \OOO_b^{1/2} )\right) \right).
\label{eq:lpboundonbreg}
\end{align}
Plugging the inequality from Equation~\eqref{eq:lpboundonbreg}  into Equation~\eqref{eq:lptightcond1reg} and using that by definition $\nu_0 = -\frac{\sigmaxi^2 \dmoment{q}^p}{n^{p/2} \sigmaxi^p}$,  we obtain from completing the square:
\begin{align}
\label{eq:lptightcond2reg}
    &\sigmaxi^2
    \left(1+  \frac{\nu_0^2}{\sigmaxi^2}  \left(1+ c_6 \rho \right) + \OOOc + c_7 (\rho^2 +  \frac{\rho \nu_0}{\sigmaxi})
 + \frac{2(\nu_0- \nu)}{\left(\frac{n \sigmaxi^2}{\dmoment{q}^2}\right)^{p/2}} \left(1+ \OOO_b^{1/2}\right) \right)\\
\geq~~&\sigmaxi^2 + (\nu^2 + \eta^2)(1+ c_1 \rho) +  c_2 \sigmaxi\rho(\nu + \eta) \\
\implies 
& c_1 \nu_0^2 (\rho^{1/2} + |\nunull|^{1/2}) +  \sigmaxi^2\OOOc + c_7(\sigmaxi^2\rho^2 + \sigmaxi\rho \nu_0) 
 \geq  (\nu-\nu_0(1+ \OOO_b^{1/2}) +c_2\sigmaxi \rho)^2  + (\eta + c_3\sigmaxi \rho)^2,
\end{align}
which implies 
\begin{equation}
    c_1 \rho^2 + \OOO_c + \nu_0^2
 \gtrsim  \nu^2  + \eta^2 ~~~\mathrm{and}~~~ \nu^2 \gtrsim \nu_0^2 +   c_2 \rho^2 + \OOO_c
\end{equation}
Hence, we can conclude the proof when plugging in the definition of $\nu_0$ and using that   $\vert \sigma^2-\sigmaxi^2\vert \lesssim \rho$, and additionally note that we can choose $\kappa_2$ in Theorem~\ref{thm:regressionlpmain} such that for any $p \in\left(1+\frac{\kappa_2}{\log\logjone{d}},2\right)$ and any $d \geq n$, $\frac{\log^{\kappa_8}d}{n} = O(\nunull^2)$.

\end{proof}

\section{Proof of Theorem~\ref{thm:classificationlp}}
\label{apx:claslpproof}

Since $\RiskC(\what) = 2 -2\frac{\langle \what, \wgt \rangle}{\norm{\what}_2}$, in order to obtain a valid upper bound for $\RiskC(\what)$, it is sufficient to lower bound $\frac{\langle \what, \wgt \rangle}{\norm{\what}_2}$.
The proof of the theorem is again divided into several steps and has essentially the same structure as the proof of Theorem~\ref{thm:regressionlpmain}.
Let again $\Xs = \{\xsui\}_{i=1}^n$ be the features containing the support of $\wgt$ (i.e., the first entries of $\xui$).

Define the set $\setSdelta:= \{ (\wone,\wtwo) : \abs{\innerprod{\wone}{\wgtp}} \geq \delta \}$ with small $\delta>0$. 
The proof consists of two major parts:

\begin{enumerate}
    \item \emph{Localization.}
    We derive a  high-probability upper bound $\loneboundclas$ (Proposition~\ref{prop:phiuppercla}), which only depends on $\xi$ and $\Xs$, on the norm of the max-$\ell_p$-margin interpolator $\what$ ,
    by finding $\loneboundclas>0$ such that
    \begin{align} \label{eq:Phi_N} 
       \PhiCN :=  \|\hat{w}\|_p^p = \min_{\substack{ \forall i:\:y_i \langle \xui, w \rangle \geq 1 \\
             w \in \setSdelta} } & \norm{w}_p^p
        ~
        \leq \loneboundclas.
\end{align}

    \item \emph{Uniform convergence.}
    Conditioning on $\xi$ and $\Xs$, 
    we derive high-probability uniform upper bound  on the classification error for all interpolators $w$ with  $\norm{w}_p^p \leq \loneboundclas$.
    More precisely, we find a high-probability lower bound (Proposition~\ref{prop:unifconvcla}) for
    \begin{align} 
      \label{eq:Phi_-} 
       \PhiCm := \min_{\substack{
            \norm{w}_p^p \leq \loneboundclas \\ 
             \forall i:\:y_i \langle \xui, w \rangle \geq 1\\
             w \in \setSdelta}
        } & \frac{\langle w, \wgt\rangle}{\norm{w}_2}
        ,
    \end{align}
    which in turn then gives us a high probability upper bound for the classification error using that $\RiskC(\what) = 2 -2\frac{\langle \what, \wgt \rangle}{\norm{\what}_2}$. 
\end{enumerate}

\begin{remark}
The constraint $(\wone,\wtwo)\in \setSdelta$ in the definitions of $\PhiCN, \PhiCm$ is needed to ensure that the optimization objective $\frac{\langle \wone, \wgtp \rangle }{\sqrt{\norm{\wone}_2^2+\norm{\wtwo}_2^2}} $ is continuous. Since we can choose $\delta$ arbitrarily small, we can neglect the constraint $w \in \setSdelta$ as long as, with high probability, the set of feasible points satisfying the other constraints in $\PhiCN, \PhiCm$ does not contain an open ball around $(0,0)$. 
\end{remark}

\subsection{Application of the (C)GMT}
\label{apx:subsec:gmtclas}

While we only need bounds for $\PhiCN$ and $\PhiCm$ for the proof of Theorem~\ref{thm:classificationlp}, for completeness of the following Proposition~\ref{prop:CGMT_application_classification}, we also define
\begin{align}
       \max_{\substack{
            \norm{w}_p^p \leq \loneboundclas \\ 
             \forall i:\:y_i \langle \xui, w \rangle \geq 1\\
             w \in \setSdelta}
        } & \frac{\langle w, \wgt \rangle}{\norm{w}_2} 
        =: \PhiCp 
\end{align}
and the function $\fnclas : \RR^{s}\times\RR_+ \to \RR_+$, 
    \begin{align}\label{eq:fnclas_def}
    \fnclas(\wone,\norm{\wtwo}_2) = \frac{1}{n} \sumin \possq{1-\xi_i\sgn(\langle  \xsui,\wgtp \rangle) \langle \xsui, \wone\rangle -\Ggaussian_i \norm{\wtwo}_2}, 
    \end{align}
     with  $\Ggaussian \sim \NNN(0,I_n)$. 
Using the same notation as for the regression setting (Proposition~\ref{prop:CGMT_applicationreg}), and by following a standard argument relying on the Lagrange multiplier, we can bring $\PhiCN,\PhiCp$ and $\PhiCm$ in a suitable form to apply the (C)GMT (see e.g. \cite{deng21} for a similar application), we have
\begin{proposition}
\label{prop:CGMT_application_classification}
    Let $p \geq 1$ be some constant and assume that $\wgt$ is s-sparse. 
   Further, let $\Hgaussian \sim \NNN(0,I_{d-s})$ and let $\delta>0$ be an arbitrary constant. Define the stochastic auxiliary optimization problems:\footnote{We define $\PhiCN,\PhiCm,\PhiCdn,\PhiCdm = \infty$ and $\PhiCp,\PhiCdp = -\infty$ if the corresponding optimization problems have no feasible solution.}
        \begin{align}
         \label{eq:phinreg}
            \PhiCdn &= \min_{(\wone,\wtwo)} \norm{\wone}_p^p+\norm{\wtwo}_p^p
            \subjto \frac{1}{n} \innerprod{\wtwo}{\Hgaussian}^2 \geq  \fnclas(\wone,\|\wtwo\|_2) \\
         \label{eq:phipreg}
            \PhiCdp &= \max_{(\wone,\wtwo)\in \setSdelta} \frac{\langle \wone, \wgtp \rangle }{\sqrt{\norm{\wone}_2^2+\norm{\wtwo}_2^2}} 
            \subjto \begin{cases}
                \frac{1}{n} \innerprod{\wtwo}{\Hgaussian}^2 \geq  \fnclas(\wone,\|\wtwo\|_2)\\
              \norm{\wone}_p^p+\norm{\wtwo}_p^p \leq \loneboundclas
            \end{cases}\\
        \label{eq:phimreg}
            \PhiCdm &= \min_{(\wone,\wtwo)\in \setSdelta} \frac{\langle \wone, \wgtp \rangle }{\sqrt{\norm{\wone}_2^2+\norm{\wtwo}_2^2}}  
            \subjto \begin{cases}
               \frac{1}{n} \innerprod{\wtwo}{\Hgaussian}^2 \geq  \fnclas(\wone,\|\wtwo\|_2) \\
               \norm{\wone}_p^p+\norm{\wtwo}_p^p  \leq \loneboundclas
            \end{cases}
        \end{align}
    where $\loneboundclas>0$ is a constant possibly depending on $\xi$ and $\Xs$. Then for any $t \in \RR$,  we have:
        \begin{align}
            \PP( \PhiCN > t \vert \xi, \Xs) &\leq 2\PP( \PhiCdn \geq t \vert \xi, \Xs)
             \\
            \PP( \PhiCp > t \vert \xi, \Xs) &\leq 2\PP( \PhiCdp \geq t \vert \xi, \Xs) 
             \\
            \PP( \PhiCm < t \vert \xi, \Xs) &\leq 2\PP( \PhiCdm \leq t \vert \xi, \Xs) ,
        \end{align}
        where the probabilities on the LHS and RHS are over the draws of $X''$ and of $\Ggaussian,\Hgaussian$, respectively.
\end{proposition}   

The proof of the proposition is presented in Appendix \ref{apx:subsec:cgmtproof}. \\

\subsection{Localization step - bounding \texorpdfstring{$\PhiCdn$}{phiN}}

\label{apx:subsec:localisationproofcla}
Similar to proof of Theorem~\ref{thm:regressionlpmain} in Appendix~\ref{apx:subsec:localisationproofreg}, we first derive an upper bound for $\PhiCdn$, where we again assume for simplicity that $\wgt, x_i \in \RR^{d+1}$ and thus $H \in \RR^d$ (see Remark \ref{rm:simplicification}). In particular, using that $\wgt = (1,0,\cdots,0)$ we can rewrite the optimization problem defining $\PhiCdn$ in Proposition~\ref{prop:CGMT_application_classification} as:
\begin{equation}
  \PhiCdn = \min_{(\wone,\wtwo)} \vert \wone\vert^p+\norm{\wtwo}_p^p
            \subjto \frac{1}{n} \innerprod{\wtwo}{\Hgaussian}^2 \geq  \fnclas(\wone,\|\wtwo\|_2),
    \end{equation}
with  $\fnclas(\wone,\norm{\wtwo}_2) = \frac{1}{n} \sumin \possq{1-\wone \xi_i \vert \xsui\vert  -\norm{\wtwo}_2 \Ggaussian_i }$. 

Before stating the proof, we first introduce some additional notation. We define $\nu := w'$ and $\eta = \norm{\wtwo}_2$, and define
\begin{align}
\label{eq:defineg1}
    \fnstar &:= \min_{\nu} \fnclas(\nu,0) ~~\mathrm{and}~~ \nubarn := \argmin_{\nu} \fnclas(\nu,0)~~\mathrm{and}\\
     \cfnone &:= \frac{\partial^2}{\partial\nu^2}\fnclas(\nubarn,0) = \frac{2}{n}\sum_{i=1}^n \idvec[ 1- \xi_i\nubarn \abs{\gausswone_i}] (\gausswone_i)^2 ~~\mathrm{and}~~ \cfntwo := \frac{2}{n}\sum_{i=1}^n \idvec[ 1- \xi_i\nubarn \abs{\gausswone_i}].
     \label{eq:defineg2}
\end{align}
Furthermore, we define $\fclas (\nu,\eta) := \EE_{Z_1,Z_2 \sim \NNN(0,1)} \EE_{\xi \sim \prob(\cdot\vert Z_1)} \left( 1- \xi \nu \vert Z_1 \vert - \eta Z_2\right)_+^2 $ and $\nubar := \arg\min \fclas(\nu, 0)$ and define the quantities ${\fstar := \fclas(\nubar, 0)}$, $\feestar := \frac{d^2}{d^2 \eta}\vert_{(\nu, \eta) = (\nubar, 0)} \fclas(\nu,\eta)$ and $\fnnstar := \frac{d^2}{d^2 \nu}\vert_{(\nu, \eta) = (\nubar, 0)} \fclas(\nu,\eta)$, which are all non-zero positive constants independent of $n,d$ and $p$ as they only depend on $\probsigma$. Throughout the proof, we implicitly make use of the fact that, with high probability over the draws of $X'$ and $\xi$, the following lemma holds.

\begin{lemma}
\label{lm:quantconcentration}
There exists universal constants $c_1,c_2>0$ such that with probability  $\geq 1-c_1 d^{-c_2}$ over the draws of $\gausswone$ and $\xi$, we have that
\begin{align}
    \max\left(\vert \nubar - \nubarn \vert, \vert \fstar - \fnstar\vert, \vert \cfnone - \fnnstar \vert, \vert \cfntwo -\feestar \vert \right) \lesssim \left(\frac{\log d}{n}\right)^{1/4}
\end{align}
\end{lemma}

\begin{proof}
Both $\fnclas$ and $\fclas$ are convex functions with unique global minima with high probability. In particular, because of Lemma~\ref{quadratic_bound_fn} we know that $\nubarn$ is attained in a bounded domain around zero with probability $\geq 1-c_1 d^{-c_2}$. Further, by Lemma~\ref{lemma_quadratic_plus_fstar_bound} we see that, with high probability, $\vert \fnstar - \fstar \vert \lesssim \frac{\log d}{\sqrt{n}}$ and $\vert \nubar - \nubarn \vert \lesssim \frac{\log^{1/2} d}{n^{1/4}}$. 

To control the quantities $\cfnone,\cfntwo$,
note that we can use exactly the same uniform convergence argument as for Proposition~\ref{prop:unifConv} to show that the functions $\nu \mapsto  \frac{2}{n}\sum_{i=1}^n \idvec[ 1- \xi_i\nu \abs{\gausswone_i}] (\gausswone_i)^2$ and  $\nu \mapsto  \frac{2}{n}\sum_{i=1}^n \idvec[ 1- \xi_i\nu \abs{\gausswone_i}]$ also converge uniformly for any bounded domain $\nu \in [\nubar - \delta, \nubar + \delta]$ with $\delta >0$. Thus, the proof then follows from  $\nubarn \to \nubar$  and the fact that 
$h_1(\nu) = \EE_{\gausswone,\xi} 
2\idvec[ 1- \xi \nu \abs{\gausswone}] (\gausswone)^2$ and $h_2(\nu) = \EE_{\gausswone,\xi} 
2\idvec[ 1- \xi\nu \abs{\gausswone}]$ are both Lipschitz continuous functions
in the domain $[\nubar - \delta, \nubar + \delta]$ with $h_1(\nubar) = \fnnstar$, $h_2(\nubar) = \feestar$.
\end{proof}

The goal of this section is to show the following upper bound on $\PhiCdn$:
\begin{proposition}
\label{prop:phiuppercla}
    Let the data distribution be as described in Section~\ref{subsec:clasetting} and assume that the noise model $\probsigma$ is independent of $n,d$ and $p$.
Under the same assumptions as in Theorem~\ref{thm:classificationlp} for $n,d,p$, there exists universal constants $c_1,c_2,\cdots,c_8>0$ such that with probability at least $ 1- c_1 d^{-c_2}$ over the draws of $\Ggaussian,\Hgaussian,\Xs,\xi$, it holds that
\begin{align}
\label{eq:upperboundlpcla}
\PhiCdn \leq&  \left(\frac{n \fnstar}{\dmoment{q}^2}\right)^{p/2}\left(1+ \frac{p}{2} \frac{\cfnone \deltanu_0^2}{2\fnstar}\left(1 + \OOOc^{1/2} \log^{3/2} d  + c_6 \left(
\abs{\deltanu_0} \log^{3/2} d +  \log d \rho\right)\right) + \OOOc +c_7   \rho^2 \right)\\
&+ \nubarn^{p} + p\nubarn^{p-1}\deltanu_0(1+ c_8\deltanu_0) =: \loneboundclas,
\end{align}
with $\OOOc = c_5 \frac{ n \exp(c_3 q)}{d q } $,  $\rho = \frac{\log^{c_4} d}{\sqrt{n}}$ and $\deltanu_0 := -\frac{ 2\nubarn^{p-1} \fnstar \dmoment{q}^p}{\cfnone \left(n \fnstar\right)^{p/2}}$.
\end{proposition}
\begin{proof}
The proof essentially follows again from exactly the same argument as used in the proof of Proposition~\ref{prop:phidnhighreg}, where the goal is again to find one feasible point. We choose $\wtwo = b \wgrad = b \partial \|H\|_q$ and $\wone = \nubarn + \deltanu_0$. 
Conditioning on the event where the lower bound for $\fnclas$ in Equation~\eqref{eq:lpfn} from Lemma~\ref{lm:lpfn} holds, we obtain the following lower bound on $b$ (we choose again  $\wtwo = b \wgrad = b \partial \|H\|_q$):

\begin{align}
&\frac{1}{n} 
\innerprod{\wtwo}{\Hgaussian}^2 \geq \fnclas(\wone,\|\wtwo\|_2 ) = \fnclas(\nubarn + \deltanu_0,\|\wtwo\|_2 ) \\
\impliedby
&b^2 \frac{\norm{H}_q^2}{n} \geq  \fnstar+ ( \frac{\cfnone}{2} \deltanu_0^2+b^2 \frac{\cfntwo}{2} \|\wgrad\|_2^2)(1+c_1\rho) + c_2\rho b\|\wgrad\|_2 \\ &~~~~~+c_3(\rho + \sqrt{\log d}(\vert \deltanu_0 \vert + b\|\wgrad\|_2))^3
\\ 
\impliedby
&b^2 \geq \frac{ \fnstar + \frac{\cfnone}{2} \deltanu_0^2(1+c_1\rho) +  c_2\rho b\|\wgrad\|_2  + c_3(\rho + \sqrt{\log d}(\vert \deltanu_0\vert + b\|\wgrad\|_2))^3 }{\frac{\norm{H}_q^2}{n} - \frac{\cfntwo}{2} (1 + c_1 \rho ) \norm{\wgrad}_2^2 },\\ \label{eq:lpboundphinbclas} 
\end{align}

where we can again choose $\rho = \frac{\log^{c} d}{\sqrt{n}}$ with $c>1$ some universal constant. We can now apply the concentration inequalities for $\|H\|_q$ and $\norm{\wgrad}_2^2$ from Lemma~\ref{lm:dualnormconc} to show that 

\begin{align}
b^2 =  \frac{n \fnstar}{\dmoment{q}^2} \left(1+\frac{\cfnone \deltanu_0^2}{2\fnstar}\left(1 + \OOOc^{1/2} \log^{3/2} d + c_6\left(
\vert \deltanu_0\vert \log^{3/2} d + \log d \rho \right)  \right)  + \OOOc + c_7  \rho^2 \right)
\end{align}
satisfies Equation~\eqref{eq:lpboundphinbclas}. Unlike in the proof of Proposition~\ref{prop:phidnhighreg}, we used that 
\begin{align} \OOOc = O(\log^{-5} d)
 ~\mathrm{and}~
 \deltanu_0^2 = \Theta \left( \frac{q^p d^{2p-2}}{n^p} \right) = O(\log^{-5}d), \label{eq:facteq2clas}
\end{align}
which is again satisfied when applying Proposition~\ref{cor:hqbound2} and choosing  $\kappa_1,\cdots, \kappa_4 >0$ characterizing $n,d$ and $p$ in Theorem~\ref{thm:classificationlp} appropriately. Hence, we can conclude the proof.

\end{proof}

\subsection{Uniform convergence step}
\label{apx:subsec:clauniformproof}

We use the same notation as in Appendix~\ref{apx:subsec:localisationproofcla} and Appendix~\ref{apx:subsec:localisationproofreg}. 
Similarly to Appendix~\ref{apx:subsec:unfiformconvreg}, we can again relax the constraints in the optimization problem defining $\PhiCdm$ in Proposition~\ref{prop:CGMT_application_classification} by using H\"olders inequality $\langle w, \Hgaussian \rangle \leq \|w\|_p \|\Hgaussian\|_q $. Note that both problems again only depend on $\wtwo$ via $\norm{\wtwo}_2$ and $\norm{\wtwo}_p$. 
Define 
\begin{equation}
    \Gamma = \{ (\nu, \eta)\vert \exists b>0 \subjto \frac{1}{n} \norm{H}_q^2 b^2  \geq  \fnclas(\nu, \eta) ~~\mathrm{and}~~ \nu^p + b^p \leq M \}
\end{equation}. It is again straight forward to verify that 
\begin{align}
        \PhiCdm \geq \left[\min_{(\nu,\eta) \in \Gamma } \frac{\nu }{\sqrt{\nu^2+\eta^2}} 
                \right]\geq \left( 1 + \frac{\max_{(\nu,\eta) \in \Gamma}\eta^2}{\min_{(\nu,\eta) \in \Gamma}\nu^2}\right)^{-1/2},
    \end{align}
where the last inequality holds when  $\min_{\nu \in \Gamma}\nu >0$. 
The goal of this section is to prove the following proposition, from which the theorem then straightforwardly follows. 
\begin{proposition}
\label{prop:unifconvcla}
    Let the data distribution be as described in Section~\ref{subsec:clasetting} and assume that the noise model $\probsigma$ is independent of $n,d$ and $p$.
Under the same conditions as in Theorem~\ref{thm:classificationlp} for $n,d$ and $p$ and for the choice of $\loneboundclas$ as in Proposition~\ref{prop:phiuppercla},  there exists universal constants $c_1,c_2,c_3,c_4,c_5>0$ such that with probability at least $ 1- c_1 d^{-c_2}$ over the draws of $\Ggaussian,\Hgaussian,\Xs,\xi$, it holds that
\begin{align}
        \Gamma \subseteq \left\{ (\nu,\eta) \in \mathbb{R}\times \mathbb{R}_+\vert (\nu - \nubarn)^2 \lesssim  \OOOc +\deltanu_0^2 +  \rho^2  ~~\mathrm{and}   ~~
   \eta^2 \lesssim \OOOc + \deltanu_0^3 \log^{3/2} d  +\rho^2 \right\}
\end{align}
with $\OOOc = c_5 \frac{ n \exp(c_3 q)}{d q } $,  $\rho = \frac{\log^{c_4} d}{\sqrt{n}}$ and $\deltanu_0 := -\frac{ 2\nubarn^{p-1} \fnstar \dmoment{q}^p}{\cfnone \left(n \fnstar\right)^{p/2}}$.
\end{proposition}

\begin{proof}
The proof is very similar to the proof of the uniform convergence bound in Appendix~\ref{apx:subsec:unfiformconvreg}, and consists of iteratively bounding $\eta+ \nu$. The primary difference lies in the control of $\fnclas$, which requires more care.

\paragraph{Step 1: $\Gamma \subset \{ (\nu,\eta)\vert \nu^2+ \eta^2 = O(1) \}$} 
A first  bound is obtained from the relaxation $b^p \leq \loneboundclas$, which implies that ${\frac{1}{n} \norm{H}_q^2 \loneboundclas^{2/p}  \geq  \fnclas(\nu, \eta)}$.

Conditioning on the event where Equation~\eqref{eq:quadratic_bound_fn} in Lemma \ref{quadratic_bound_fn} holds, we have that $\fnclas(\nu, \eta) \geq c_1 \nu^2 + c_2 \eta^2$, and thus we can relax
\begin{equation}
    \frac{\loneboundclas^{2/p}\|\Hgaussian\|_q^2}{n} \geq  c_1 \nu^2 + c_2 \eta^2
\end{equation}
We obtain the desired upper bound from the concentration of $\|\Hgaussian\|_q^2$ from Lemma~\ref{lm:dualnormconc}  and when noting that  the dominating term in $\loneboundclas$ is $ \left(\frac{n \fnstar}{\dmoment{q}^2}\right)^{p/2}$:
\begin{align}
\label{eq:bound1lpeq}
        \nu^2 + \eta^2 = O(\fnstar) = O(1). 
\end{align}

\paragraph{Step 2:  $\Gamma \subset \{ (\nu,\eta) \vert (\nu - \nubarn)^2+ \eta^2 = O(\log^{-5}d) \}$ }
Define $\deltanu := \nu -  \nubarn$.
Conditioning on Equation~\eqref{eq:bound1lpeq}, we can now repeat the same argument but using the tighter lower bound for $\fnclas$ from Equation~\eqref{eq:quadratic_plus_fstar_bound} in Lemma~\ref{lemma_quadratic_plus_fstar_bound}, where we let $\epsilon$ in Lemma~\ref{lemma_quadratic_plus_fstar_bound} be $ \frac{\log^{c} d}{\sqrt{n}} =: \rho$ with universal constant $c>1$. Thus, we obtain the high probability upper bound:
\begin{align}
    \frac{\loneboundclas^{2/p}\|\Hgaussian\|_q^2}{n} \geq \fnclas(\nu,\eta) &\geq  \fstar + \tconstnu  (\nu - \nubar)^2+ \tconsteta \eta^2 + c_1\rho 
    \\ &=  \fstar + \tconstnu  (\deltanu^2 + 2 \deltanu(\nubarn - \nubar) +  (\nubarn - \nubar)^2) + \tconsteta \eta^2 + c_1\rho,\\
    &\geq \fnstar + \tconstnu  \deltanu^2 +  \tconsteta \eta^2 + c_2 (\rho^{1/2} + \deltanu \rho^{1/2}).
    \label{eq:lpclastep2}
\end{align}
In particular, we can take the Taylor series approximation of $\loneboundclas^{2/p}$ which gives us
\begin{align}
    \loneboundclas^{2/p} =  \frac{n \fnstar}{\dmoment{q}^2} \left(1 +\OOOc + O\left(\frac{\cfnone \deltanu_0^2}{2\fnstar} +  \rho^2 + \frac{\nubarn^{p}}{ \left(\frac{n \fnstar}{\dmoment{q}^2}\right)^{p/2}} \right)\right).
\end{align}
And hence, applying again  Lemma~\ref{lm:dualnormconc} to control $\norm{H}_q^2$ as in Appendix~\ref{apx:subsec:localisationproofreg}, the bound in Equation~\eqref{eq:lpclastep2} implies 
\begin{equation}
\label{eq:bound2lpeqclas}
  \tconstnu \deltanu^2 + \tconsteta \eta^2 = \OOOc + O \left(\left(\frac{\dmoment{q}^2}{n \fnstar}\right)^{p/2} + \rho^{1/2} \right) = O\left(\log^{-5}d\right),
\end{equation}
where in the last line we used Equation~\eqref{eq:facteq2clas}.

\paragraph{Step 3: $\Gamma \subset \{ (\nu,\eta) \vert (\nu - \nubarn)^2 \lor \eta^2 = \OOOc +O(\deltanu_0^2 +  \rho^2)$}
We are now able to derive a tighter bound 
conditioning on the event where Equation~\eqref{eq:bound2lpeqclas} holds and thus $\Gamma \in \{(\nu,\eta) \vert\ \vert \nu - \nubarn\vert^2 \lesssim \log^{-5}d  ~~\mathrm{and}~~ \eta^2 \lesssim \log^{-5}d\}$.  We can apply  Lemma~\ref{lm:lpfn}, which allows us, with high probability, to uniformly bound
\begin{equation}
   \left| \fnclas(\nu, \eta)  -  \fnstar -\deltanu^2 \frac{\cfnone}{2} - \eta^2 \frac{\cfntwo}{2} \right| \lesssim (\deltanu^2 + \eta^2) \log^{-1}d + \eta \rho + \rho^3
\end{equation}
where we choose again $\rho = \frac{\log^{c}d}{\sqrt{n}}$. 
Furthermore, we can relax the  second constraint in $\Gamma$ to
\begin{align}
    b^p &\leq \loneboundclas - (\deltanu+\nubarn)^p =  \left(\frac{n \fnstar}{\dmoment{q}^2}\right)^{p/2}\left(1+ \OOOc +c_1(\deltanu_0^2 +  \rho^2) + \frac{p\nubarn^{p-1}}{ \left(\frac{n \fnstar}{\dmoment{q}^2}\right)^{p/2}} c_2\left(\deltanu_0-\deltanu\right)\right)
    \\
    &=  \left(\frac{n \fnstar}{\dmoment{q}^2}\right)^{p/2}\left( 1+ \OOOc +c_1(\deltanu_0^2 +  \rho^2) +  c_3\left(\deltanu_0^2-\deltanu_0\deltanu\right)\right
    )\\
    &=  \left(\frac{n \fnstar}{\dmoment{q}^2}\right)^{p/2}\left(1+  \OOOc +c_4(\deltanu_0^2 +  \rho^2  + \deltanu_0\deltanu)\right)\label{eq:lpregtighter}
\end{align}
 with $\loneboundclas$ from Proposition~\ref{prop:phiuppercla} and where we used that by definition $\deltanu_0 = -\frac{ 2\nubarn^{p-1} \fnstar}{\cfnone \left(\frac{n \fnstar}{\dmoment{q}^2}\right)^{p/2}}$. 
In summary, we can again obtain an upper bound for $\eta^2$ and $\deltanu^2$ from $  \frac{1}{n} \norm{H}_q^2 b^2 \geq \fnclas(\nu,\eta)$ and when following the same argument as in Appendix~\ref{apx:subsec:localisationproofreg}. We have  
\begin{align}
     &\frac{1}{n} \norm{H}_q^2 b^2 \geq \fnclas(\nu,\eta)\\
     \implies ~~&\frac{\fnstar}{ \frac{n \fnstar}{\dmoment{q}^2}}(1+\epsilon)  \frac{n \fnstar}{\dmoment{q}^2}\left(1+  \OOOc + c_4(\deltanu_0^2 +  \rho^2  + \deltanu_0\deltanu)\right) \\
     &\geq   \fnstar + \deltanu^2 \frac{\cfnone}{2} + \eta^2 \frac{\cfntwo}{2}
     +c_5((\deltanu^2 + \eta^2) \log^{-1}d + \rho\eta + \rho^3)\\
     \implies~~&\OOOc + c_4(\deltanu_0^2 + \rho^2  + \deltanu_0\deltanu) \geq 
      \deltanu^2 \frac{\cfnone}{2} + \eta^2 \frac{\cfntwo}{2}
     +c_5((\deltanu^2 + \eta^2) \log^{-1}d + \rho\eta + \rho^3)\\
     \implies   ~~&\deltanu^2 \lor \eta^2 = \OOOc +O(\deltanu_0^2 +  \rho^2 ),
     \label{eq:lpclassstep3}
\end{align}
where we used in the second line that, with high probability, by Lemma~\ref{lm:dualnormconc}, $\frac{\norm{H}_q}{n} \leq \frac{\fnstar}{ \frac{n \fnstar}{\dmoment{q}^2}}(1 + \epsilon)$ with $\epsilon \lesssim \OOOc + \rho^2$.


\paragraph{Step 4: Bound in Proposition~\ref{prop:unifconvcla}} We are now ready to prove the bounds in Proposition~\ref{prop:unifconvcla}.
We already know from the previous steps that $\deltanu \to 0$ and hence $\nu$ concentrates around $\nubarn$, which itself concentrates around $\nubar$ by Lemma~\ref{lm:quantconcentration}. However, the classification error in Theorem~\ref{thm:classificationlp} depends on the term $\frac{\eta^2}{\nu^2}$, which allows us to obtain a tighter bound when further bounding $\eta^2$. 

 Lemma~\ref{lm:lpfn} together with the previous bound on $\eta^2$ and $\deltanu^2$ implies that uniformly over all $\deltanu^2 \lor \eta^2 = \OOOc + O(\deltanu_0^2 +  \rho^2 )$, we have:
\begin{align}
    \left| \fnclas(\nu, \eta) -  \fnstar - \deltanu^2 \frac{\cfnone}{2} - \eta^2 \frac{\cfntwo}{2} \right| &\lesssim (\deltanu^2+\eta^2) \OOO_b + \rho^2  + \rho \eta,
\end{align}
where we define $\OOO_b =  c_1 \log^{3/2}d~(\sqrt{ \OOOc} +  O\left( \rho + \deltanu_0  \right))$. 

Instead of Equation~\eqref{eq:lpregtighter}, we can analogously obtain the tighter expression (using that $\deltanu = \OOO_b$)
\begin{align}
\label{eq:lptightcond2clas}
    b^p &\leq \loneboundclas - (\deltanu+\nubarn)^p \\&\leq  \left(\frac{n \fnstar}{\dmoment{q}^2}\right)^{p/2}\left(1+ \frac{p}{2} \frac{\cfnone\deltanu_0^2}{2\fnstar}(1+ \OOO_b)
    +\OOOc+ c_1 \rho^2 + \frac{p\nubarn^{p-1}}{ \left(\frac{n \fnstar}{\dmoment{q}^2}\right)^{p/2}}\left(\deltanu_0-\deltanu\right)(1+ \OOO_b)\right).
\end{align}

Therefore, applying the Taylor series approximation and Lemma~\ref{lm:dualnormconc} as in Appendix~\ref{apx:subsec:localisationproofreg}, we can upper bound

\begin{align}
\label{eq:lptightcond2clas2}
    \frac{1}{n} \norm{H}_q^2 b^2 &\leq   \frac{1}{n} \norm{H}_q^2  \frac{n \fnstar}{\dmoment{q}^2}\left(1+ \frac{\cfnone\deltanu_0^2}{2\fnstar} (1+\OOO_b) +
  \OOOc +   c_1 \rho^2 + 
    \frac{2\nubarn^{p-1}}{ \left(\frac{n \fnstar}{\dmoment{q}^2}\right)^{p/2}}\left(\deltanu_0-\deltanu\right)(1+ \OOO_b)\right) \\
&=
(1 + \epsilon) \fnstar\left(1+ \frac{\cfnone\deltanu_0^2}{2\fnstar} (1+\OOO_b) +
  \OOOc + c_1 \rho^2 + 
    \frac{2\nubarn^{p-1}}{ \left(\frac{n \fnstar}{\dmoment{q}^2}\right)^{p/2}}\left(\deltanu_0-\deltanu\right)(1+ \OOO_b)\right) \\
&= \fnstar + \frac{\cfnone}{2} (2 \deltanu_0 \deltanu -\deltanu_0^2) (1+\OOO_b) +
   \OOOc + c_1 \rho^2 ,
   \label{eq:lpboundonb}
\end{align}
where we used the same argument as in the previous step.
As a result, we obtain the upper bound 
\begin{align}
\label{eq:lptightcond2clas3}
&\frac{1}{n}b^2\norm{H}_q^2 \geq \fnclas(\nu,\eta) \\
\implies~~&
\fnstar + \frac{\cfnone}{2} (2 \deltanu_0 \deltanu -\deltanu_0^2) (1+\OOO_b) +
   \OOOc + c_1\rho^2
   \\&\geq   \fnstar + \deltanu^2 \frac{\cfnone}{2} + \eta^2 \frac{\cfntwo}{2}  + (\deltanu^2+\eta^2)  \OOO_b  +c_2\rho^2 + c_3\rho \eta\\
  \implies ~~& \deltanu_0^2\OOO_b + \OOOc + c_1\rho^2 \geq 
   \frac{\cfntwo}{2}(\eta - c_2\rho)^2 +\frac{\cfnone}{2}(\deltanu - \deltanu_0)^2.
  \end{align}

  Finally, we get the desired result from the fact that $\cfnone$ and $\cfntwo$ concentrate around $\fnnstar$ and $\feestar$ by Lemma~\ref{lm:quantconcentration}.  Hence we can conclude the proof.

\end{proof}

\subsection{Proof of Proposition \ref{prop:CGMT_application_classification}: Application of the (C)GMT for classification}
\label{apx:subsec:cgmtproof}

The proof essentially follows exactly the same steps as in \cite{koehler2021uniform} except for a few simple modifications which we describe next. First we introduce a more general form of  the (C)GMT:

\begin{lemma}
\label{lemma:CGMT_variant}
Let $X'' \in \RR^{n\times {d-s}}$ be a matrix with i.i.d. $\NNN(0,1)$ entries and let $\Ggaussian \sim \NNN(0,I_n)$ and $H \sim \NNN(0,I_{d-s})$ be independent random vectors. Let $\compactw \subset \RR^{s}\times \RR^{d-s}$ and $\compactv \subset \RR^{n}$ be compact sets, and let $\psi: \compactw \times \compactv \to \RR$ be a continuous function. Then for the following two optimization problems:
    \begin{equation}
        \Phi = \min_{(\wone,\wtwo)\in \compactw}\max_{v\in \compactv} \innerprod{v}{X'' \wtwo} + \psifcn((\wone,\wtwo),v)
    \end{equation}
    \begin{equation}
        \phi = \min_{(\wone,\wtwo)\in \compactw}\max_{v \in \compactv} \norm{\wtwo}_2\innerprod{v}{G} + \norm{v}_2\innerprod{\wtwo}{H}+\psifcn((\wone,\wtwo),v)
    \end{equation}
    and any $t\in\RR$ holds that:
    \begin{equation}
        \prob(\Phi < t)\leq 2\prob(\phi \leq t) 
    \end{equation}
    
    If in addition $\psi$ is convex-concave function we also have for any $t\in\RR$:
    \begin{equation}
        \prob(\Phi > t)\leq 2\prob(\phi \geq t) 
    \end{equation}
    In both inequalities the probabilities on the LHS and RHS are over the draws of $X''$, and of $G$, $H$, respectively.
\end{lemma}
\begin{proof}
The first part of the lemma is equivalent to Theorem 10 in \cite{koehler2021uniform}. The proof of the second part follows from Theorem 9 in \cite{koehler2021uniform} and proof of Theorem 10 in \cite{koehler2021uniform}.
\end{proof}

We first rewrite $\PhiCN$  using the Lagrange multipliers $\lmult\in\RR^n$ to be able to apply Lemma \ref{lemma:CGMT_variant}:
\begin{align}
    \PhiCN 
      &= \min_{w} \max_{\lmult\geq 0} \norm{w}_p^p + \innerprod{\lmult}{\onevec - \matDy\matX w}
    &= \min_{(\wone,\wtwo)} \max_{\lmult\geq 0} \norm{\wone}_p^p + \norm{\wtwo}_p^p + \innerprod{\lmult}{\onevec - \matDy\matXs \wone} - \innerprod{\matDy\lmult}{\matXsc \wtwo}
\end{align}
where $\matDy = \mathrm{diag}(y_1,y_2,\dots,y_n)$ and $\matXsc$ is the sub-matrix containing the  $d-s$ columns of $X$ from the complement of the support of $\wstar$.
We  note that $\matDy\matX \in \RR^{n\times d}$ has i.i.d. entries distributed according to the standard normal distribution and that the random matrices $-\matXsc$ and $\matXsc$ have the same distribution. Also, the function $\psi((\wone,\wtwo),\lmult) := \norm{\wone}_p^p + \norm{\wtwo}_p^p + \innerprod{\lmult}{ \onevec - \matDy \matXs \wone}$ is a continuous convex-concave function on the whole domain. 
We further define
\begin{align}
    \PhiCdn &= \min_{(\wone,\wtwo)} \max_{\lmult\geq 0} \norm{\wone}_p^p + \norm{\wtwo}_p^p + \innerprod{\lmult}{\onevec - \matDy\matXs \wone} - \norm{\wtwo}_2 \innerprod{\matDy \lmult}{ \Ggaussian} - \norm{\matDy\lmult}_2 \innerprod{\wtwo}{\Hgaussian}\\
    &= \min_{(\wone,\wtwo)} \max_{\lambda\geq 0} \norm{\wone}_p^p + \norm{\wtwo}_p^p - \lambda \left(\innerprod{\wtwo}{\Hgaussian}-\norm{\pos{\onevec - \matDy\matXs \wone- \matDy \Ggaussian\norm{\wtwo}_2}}_2 \right) \\
    & = \min_{(\wone,\wtwo)} \norm{\wone}_p^p + \norm{\wtwo}_p^p
    \subjto \innerprod{\wtwo}{\Hgaussian}\geq \norm{\pos{\onevec - \matDy\matXs \wone- \matDy \Ggaussian\norm{\wtwo}_2}}_2
\end{align}
where in the second equality we set $\lambda:= \norm{\lmult}_2$. Since $\norm{\wtwo}_p^p$ and $\norm{\wtwo}_2$ do not depend on the signs of the entries of $\wtwo$, any minimizer $w = (\wone, \wtwo)$ of $\PhiCdn$ satisfies $\innerprod{\wtwo}{H}\geq 0$. Hence squaring the last inequality and scaling with $\frac{1}{n}$, we note that the RHS is given by
    \begin{align}
    \frac{1}{n}\norm{\pos{\onevec - \matDy\matXs \wone- \matDy \Ggaussian\norm{\wtwo}_2}}_2^2 = \frac{1}{n} \sumin \possq{1-\xi_i\sgn(\langle \xsui, \wgtp \rangle) \langle \xsui, \wone \rangle -\Ggaussian_i \norm{\wtwo}_2}, 
    \end{align}
which is exactly the function $\fnclas(\wone, \norm{\wtwo}_2)$, as defined in Equation \eqref{eq:fnclas_def}. Hence, we obtain the desired expression for $\PhiCdn$ in Proposition~\ref{prop:CGMT_application_classification}.

In order to complete the proof of the proposition, we need to discuss
compactness of the feasible sets in the optimization problem 
so that we can apply Lemma~\ref{lemma:CGMT_variant}.
For this purpose, we define the following truncated optimization problems $\PhiCN^r(t)$ and $\PhiCdn^r(t)$ for some $r,t \geq 0$:
\begin{align}
    \PhiCN^r(t) &:= \min_{\norm{w}_p^p\leq t} \max_{\substack{\norm{v}\leq r\\ v\geq 0}} \norm{w}_p^p + \innerprod{v}{1-D_y X w}
    \\
    \PhiCdn^r(t) &:= \min_{\norm{\wone}_p^p+\norm{\wtwo}_p^p \leq t} \max_{0\leq \lambda\leq n r} \norm{\wone}_p^p + \norm{\wtwo}_p^p - \lambda \left(  \frac{1}{n}\innerprod{\wtwo}{\Hgaussian}^2 - \fnclas(\wone,\|\wtwo\|_2) \right).
\end{align}
By definition it then follow that 
\begin{equation}
        \PP ( \PhiCN > t | \xi, X') \leq \inf_{r\geq 0} \PP (\PhiCN^r(t) > t | \xi,X').
\end{equation}
Furthermore, by making use of the simple (linear) dependency on $\lambda$ in the optimization objective in the definition of $\PhiCdn$,  a standard limit argument as in the proof of Lemma 7 in \cite{koehler2021uniform} shows that:
\begin{align}
    \PP (\PhiCdn \geq t | \xi, X') \geq \inf_{r\geq 0} \PP (\PhiCdn^r(t) \geq t | \xi, X')
\end{align}
Finally, the proof follows when noting that we can apply Lemma~\ref{lemma:CGMT_variant} directly to $\PhiCN^r(t)$ and $\PhiCdn^r(t)$ for any $r,t\geq 0$, which gives us $\prob(\PhiCN^r>t|\xi,\Xs) \leq 2\prob(\PhiCdn^r\geq t|\xi,\Xs)$. Combining the previous two statements completes the proof for $\PhiCN$.



The proof for $\PhiCp$ and $\PhiCm$ uses the same steps as discussed above. We only detail the proof for $\PhiCm$ here, as the the proof for $\PhiCp$ follows from the exact same reasoning.
Let $\BBB_p(M) = \{w\in\RRR^d: \norm{w}_p^p \leq M \}$ be an $l_p$-ball of radius $M$ and note that we optimize over $(\wone,\wtwo)\in \compactw$ where $\compactw = \setSdelta \cap \BBB_p(\loneboundclas)$ is a compact set for $\delta>0$ sufficiently small. Furthermore, define the function $\psi$ by $\psi((\wone,\wtwo),v) := \frac{\innerprod{\wone}{\wgtp}}{\sqrt{\norm{\wone}_2^2+\norm{\wtwo}_2^2}} + \innerprod{v}{\onevec - \matDy \matXs \wone}$, which is a continuous function on $\compactw$. Similarly as above, we can overcome the issue with the compactness of the set $\compactv$ by using a truncation argument as proposed in Lemma 4 in \cite{koehler2021uniform}. In particular, we define
\begin{align}
    \PhiCm^r&:= \min_{\substack{w\in \setSdelta\\ \norm{w}_p^p \leq \loneboundclas}} \max_{\substack{\norm{v}\leq r\\ v\geq 0}} \frac{\langle w, \wgt \rangle}{\norm{w}_2}  + \innerprod{v}{1-D_y X w},
    \\
    \PhiCdm^r &:= \min_{\substack{(\wone,\wtwo)\in \setSdelta\\ \norm{\wone}_p^p+\norm{\wtwo}_p^p \leq \loneboundclas}} \max_{0 \leq \lambda\leq n r} \frac{\langle \wone, \wgtp \rangle }{\sqrt{\norm{\wone}_2^2+\norm{\wtwo}_2^2}} - \lambda \left( \frac{1}{n}\innerprod{\wtwo}{\Hgaussian}^2 - \fnclas(\wone,\|\wtwo\|_2) \right)
\end{align}
for which we have
\begin{align}
    \PP(\PhiCm < t| \xi,X') &\leq \inf_{r\geq 0} \PP(\PhiCm^r < t | \xi,X')\\
    \text{ and } \: \PP(\PhiCdm \leq t|\xi,X') &\geq \inf_{r\geq 0} \PP(\PhiCdm^r \leq t| \xi,X').
\end{align}
We note that the first statement again follows from the definition of $\PhiCm$,  while the second statement follows from a limit argument as in Lemma 4 in \cite{koehler2021uniform}. Finally, we conclude the proof by applying the first part of Lemma~\ref{lemma:CGMT_variant} to $\PhiCm^r$ and $\PhiCdm^r$. 


\section{Proof of Proposition~\ref{prop:uniformlowerbound}: Uniform lower bound for interpolating classifiers }
\label{apx:uniformlowerbound}



First, note that we can assume without loss of generality that $\wgt$ is normalized, i.e. $\norm{\wgt}_2^2 = 1$. Furthermore, because the classification error is invariant under rotations of the input features, we can assume without loss of generality that $\wgt = (1,0,\cdots,0)$. This trick only works for rotational invariant distributions, including the Gaussian distribution,  where the marginal distribution for every sample rotated by a orthogonal matrix is again i.i.d.\ Gaussian. More precisely, let $\tilde{X} = OX$ with $O$ an orthogonal matrix and $X \sim \NNN(0,I_d)$, then $\tilde{X} \sim \NNN(0,I_d)$ as well. 

We now distinguish between the two cases where $\langle \hat{w}, \wgt \rangle \leq 0$ and $\langle \hat{w}, \wgt \rangle > 0$. Clearly, in the first case we have that $\RiskC(\what) \geq 2$, and hence, we only need to bound the second case. Furthermore, because the risk is invariant under rescalings of the vector $\what$, we can assume without loss of generality that $\langle \what, \wgt \rangle = 1$.

For any $B >0$ define
\begin{align} 
    \Phip^{(B)} &= \max_w \frac{\innerprod{w}{\wgt}}{\norm{w}_2}
    \subjto \begin{cases}
        \min_i y_i\innerprod{x_i}{w}\geq 0\\
        \langle \wgt, w \rangle = 1 \\
        \norm{w}_2^2 \leq B,
    \end{cases}
\end{align}
where we remark that we could also choose for this proof any other norm or compact set to bound $w$ instead of the $\ell_2$-norm. 

We can now apply again the GMT (with slight trivial modifications) from Proposition~\ref{prop:CGMT_application_classification} to show that for any $t\in \mathbb{R}$, we have
\begin{equation}
                \PP( \Phip^{(B)} >t ) \leq 2\PP( \phidp^{(B)} \geq t ) 
            \label{eq:probeq3} ,
\end{equation}
with 
\begin{equation}
              \phidp^{(B)} = \max_{(\nu,\wtwo)} \frac{\nu}{\sqrt{\nu^2 +\norm{\wtwo}_2^2}}  
            \subjto \begin{cases}
               \frac{1}{n} \innerprod{\wtwo}{\Ggaussian}^2 \geq  \tilde{\fn}(\nu,\|\wtwo\|_2) \\
                 \nu^2  +\norm{\wtwo}_2^2  \leq B\\
                        \nu = 1 
            \end{cases}
\end{equation}
and $\tilde{\fn}(\nu,\eta) = \frac{1}{n} \sumin \possq{-\xi_i\sgn(\langle \wgt, \xsui \rangle) \langle \xsui, \nu\rangle -\Ggaussian_i y}$. 
Note that since we only assume that $\min_i \yui \langle \xui, \what\rangle \geq 0$ and not $\geq 1$, the constant $1$ factor in $\fnclas$ from Proposition~\ref{prop:CGMT_application_classification} in the term $(.)_+$ disappears in in $\tilde{\fn}$ . 

We can now again lower bound $\tilde{\fn}$, where we a straight forward modification of Lemma~\ref{quadratic_bound_fn} shows that, with probability at least $1-\exp(-cn)$, uniformly for all $\nu,\eta$,  we have
\begin{equation}
    \tilde{\fn}(\nu,\eta) \geq \constnu \nu^2  + \consteta \eta^2,
\end{equation} 
with $\constnu,\consteta>0$ some universal constants. Further, using Cauchy-Schwarz, we can upper bound 
$\frac{1}{n}\langle \wtwo, \Hgaussian\rangle^2 \leq \frac{1}{n}\|\Hgaussian\|_2^2 \|\wtwo\|_2^2 \leq \frac{2d \|\wtwo\|_2^2}{n}  $
 where the last inequality holds with probability at least $1- \exp(-c d)$ using standard concentration arguments. Hence, in summary we have the upper bound
 \begin{equation}
     \frac{2d \|\wtwo\|_2^2}{n} \geq \constnu \nu^2 \implies \frac{\|\wtwo\|_2^2}{\nu^2} \geq \Omega\left(\frac{n}{d}\right)
 \end{equation}
 
 Therefore, we have that for all $B>0$, $\phidp^{(B)} \leq 1 - \Omega\left(\frac{n}{d}\right)$ with universal constants independent of $B$ and thus the proof is complete.

\section{Technical lemmas: Bounds for $\fnclas$}
\label{apx:technicallemmas}


Throughout this section, for simplicity of the notation, we abbreviate $\{\xui'\}_{i=1}^n$ as $\gausswone$ and $G$ from Proposition~\ref{prop:CGMT_application_classification} as $\gausswtwo$. Thus, we have $\fnclas (\nu, \eta) := \frac{1}{n} \sumin \possq{ 1 - \xi_i \abs{\gausswone_i} \nu - \eta \gausswtwo_i }$ and  
\begin{equation}
\label{eq:defoff}
    \fclas(\nu,\eta) :=  \EE \fnclas(\nu,\eta) = \EE_{Z_1,Z_2 \sim \NNN(0,1)} \EE_{\xi \sim \probsigma(\cdot; Z_1)} \left( 1- \xi \nu \vert Z_1 \vert - \eta Z_2\right)_+^2.  
\end{equation}

\subsection{Lower bounding \texorpdfstring{$\fnclas$}{fn} by a quadratic form}
\label{apx:subsec:quadratic_bound_fn}
  We show the following lemma.
\begin{lemma} \label{quadratic_bound_fn}
    There exist universal positive constants $\constnu, \consteta$ only depending on $\probsigma$ and $c$ such that for any $\nu,\eta$ we have that: 
        \begin{align}
        \label{eq:quadratic_bound_fn}
            \fnclas(\nu,\eta) \geq \constnu \nu^2  + \consteta \eta^2
        \end{align}
    with probability at least $1 - \exp \left( -cn \right)$ over the draws of $\gausswone,\gausswtwo,\xi$.
\end{lemma}

\begin{proof}
 We can assume that $\nu ,\eta \geq 0$ as the other cases follow from exactly the same argument. 
 First we show an auxiliary statement which we use later in the proof. Namely, we claim that there exists some positive constant $c_1$ such that for all $z \in [z_1,z_2]$,  $\probsigma \left( \xi = -1; z \right) > c_1$ for some $z_1, z_2 \in \RR$ and $z_1\neq z_2$. Let's prove this statement by contradiction and assume that there exists no $z \in [z_1,z_2]$ which satisfy the previous equation. Then, for almost any $z \sim N(0,1)$, we have $\probsigma (\xi; z) = +1$ and hence the minimum of the function $\fclas(\nu, \eta) = \EE \fnclas (\nu, \eta)$ is obtained for $\nu = \infty$. However, this is in contradiction with the assumption on $\probsigma$ in Section \ref{subsec:clasetting}. Hence there exists some $z$ for which $\prob \left( \xi = -1;  z \right) > c_1$. By the assumption on $\probsigma$ in Section \ref{subsec:clasetting} we assume piece-wise continuity of $z \to \probsigma(\xi=-1;z)$ and hence there exists some interval $[z-\delta, z+\delta] =: [z_1,z_2]$ in which the given probability is bounded away from zero. 

We can assume without loss of generality that this interval does not contain zero, since in that case we can always define a new interval of the form $[\epsilon, z_2]$ or $[z_1, -\epsilon]$ for $\epsilon>0$ small enough, which does not contain zero. Let's define $\tilde{z} = \max \{ \abs{z_1}, \abs{z_2} \}$.

 We can now bound $\fnclas(\nu,\eta)$ as follows:
\begin{align}
    \fnclas(\nu,\eta) 
    & =
    \frac{1}{n} \sumin \possq{ 1 - \xi_i \abs{\gausswone_i} \nu - \eta \gausswtwo_i }
    \geq
    \frac{1}{n}\sumin \idvec \left[ \xi_i = -1, \gausswone \in [z_1,z_2], \gausswtwo < - c_2 \right] \possq{ 1 - \xi_i \abs{\gausswone_i} \nu - \eta \gausswtwo_i }
    \\
    & \geq \left( 1 + \tilde{z} \abs{\nu} + c_2 \eta \right)^2 \frac{1}{n}\sumin \idvec \left[ \xi_i = -1, \gausswone_i \in [z_1,z_2], \gausswtwo_i < -c_2 \right]
\end{align}
for arbitrary positive constant $c_2$.
From section \ref{apx:subsec:gmtclas} we have that $\gausswtwo$ is independent of $\xi$ and $\gausswone$ and hence:
\begin{align}
    \prob \left( \xi = \sign (\nu), \gausswone \in [z_1,z_2], \gausswtwo < - c_2 \right) 
    &=
    \prob \left( \xi = \sign (\nu) \vert \gausswone \in [z_1,z_2]\right) \prob \left( \gausswone \in [z_1,z_2] \right) \prob \left( \gausswtwo < - c_2 \right)
    \\
    &\geq
    c_1 \left( \Phic(z_1) - \Phic(z_2) \right) \Phic(c_2) 
    \geq
    c
\end{align}
for some positive constant $c$. Now using concentration of i.i.d. Bernoulli random variables we obtain:
\begin{align}
    \fnclas (\nu,\eta) \geq \left( 1 + \tilde{z}\abs{\nu} + c_2 \eta \right)^2 \frac{c}{2} 
    \geq
    \frac{\tilde{z}^2 c_3}{2} \nu^2 + \frac{c_2^2 c_3}{2} \eta^2
\end{align}
with universal constant $c_3>0$ and with probability at least $ 1 - \exp \left( - n D (c || \frac{c}{2} ) \right) \geq 1 - \exp \left( - n c\right)$.
\end{proof}

\subsection{Lower bounding \texorpdfstring{$\fnclas$}{fn} by a quadratic form with constant}
\label{apx:subsec:quadratic_plus_fstar_bound}
We use the notation as described in Appendix~\ref{apx:subsec:localisationproofcla} and the beginning of Appendix~\ref{apx:technicallemmas}.
 We show the following lemma. 
\begin{lemma} \label{lemma_quadratic_plus_fstar_bound}
    Let $\boundnuone, \boundetaone>0$ be two positive constants. Then, there exist positive constants $\tconstnu,\tconsteta>0$ and $c_1, c_2,c_3 >0$ only depending on $\probsigma$, such that for any $\epsilon \geq  c_1 \sqrt{\frac{\log(n)}{n}}$ and any $\nu \leq \boundnuone, \eta \leq \boundetaone$ we have that: 
        \begin{align}
        \label{eq:quadratic_plus_fstar_bound}
            \fnclas(\nu,\eta) \geq \fstar + \tconstnu \left( \deltanu \right)^2  + \tconsteta \eta^2 - \epsilon
        \end{align}
    with probability at least $1 - c_2 \exp \left( - \frac{c_3 n \epsilon^2}{\log(n)}  \right)$ over the draws of $\gausswone,\gausswtwo,\xi$.
\end{lemma}
\begin{proof}
First note that from the uniform convergence result in Proposition \ref{prop:unifConv} we have that $\fclas(\nu,\eta) \geq \fnclas(\nu,\eta) - \epsilon$, with $f$ from Equation~\eqref{eq:defoff}, with probability at least $1 - c_2 \exp \left( - \frac{c_3 n\epsilon^2}{\log(n)}  \right)$. Thus, it is sufficient to study $\fclas$.  Clearly, by the convexity of $\fclas$ we have that $\fclas \geq \fstar$ with $\fstar = \fclas(\nubar, 0)$ where we use the simple fact that $(\nubar, 0)$ is the global minimizer of $\fclas$, which follows from the  assumption on $\probsigma$ in Section \ref{subsec:clasetting}. 
Furthermore, it is not difficult to check that for for any $\nu,\eta$, $\nabla^2 \fclas(\nu,\eta) \succ 0$ and therefore, $\fclas$ is strictly convex on every compact set. Hence, the proof follows. 
\end{proof}

\subsection{Tighter bound for \texorpdfstring{$\fnclas$}{fn} around \texorpdfstring{$(\nubarn, 0)$}{(nu,0)}}
\label{apx:subsec:prooffnlp}

The bound in Lemma~\ref{lemma_quadratic_plus_fstar_bound} contains a term $\epsilon$
which is at least of order $\frac{1}{\sqrt{n}}$. This term arises from the uniform convergence of $\fnclas$ to $\fclas$ in Proposition~\ref{prop:unifConv} and cannot be avoided. Instead, we show that $\fnclas$ uniformly converges much faster to a different expression. 

\begin{lemma}
\label{lm:lpfn}
 With probability at least $1-c_1d^{-c_2}$ over the draws of $\gausswone,\gausswtwo,\xi$, we have that 
\begin{equation}
\label{eq:lpfn}
    \sup_{ \eta\geq0, \deltanu} \left\| \fnclas(\nu, \eta) -  \fnstar - \deltanu^2 \frac{\cfnone}{2} - \eta^2 \frac{\cfntwo}{2} \right\| \lesssim (\deltanu^2 + \eta^2) \rho +   \rho \eta +(\rho + \sqrt{\log(d)}(\vert\deltanu\vert + \eta))^3,
\end{equation}
with  $\rho = \frac{\sqrt{\log(d)}}{\sqrt{n}}$.
\end{lemma}

\begin{proof}
We write
\begin{align}
    \fnclas(\nu,\eta) =& \frac{1}{n} \sum_{i=1}^n \idvec[ 1- \xi_i\nu \abs{\gausswone_i} - \eta \gausswtwo_i  ]( \underbrace{ ( 1- \xi_i\nubarn \abs{\gausswone_i} )^2}_{=:T_{1,i}} +  \underbrace{\xi_i^2 \deltanu^2 \abs{ \gausswone_i }^2}_{=:T_{2,i}}+  \underbrace{\eta^2 (\gausswtwo_i)^2 }_{=:T_{3,i}} 
    \\
    &  \underbrace{-2 \xi_i \deltanu \abs{ \gausswone_i } (1- \xi_i\nubarn \abs{ \gausswone_i } )}_{T_{4,i}} \underbrace{ -  2 \eta \gausswtwo_i (1- \xi_i\nubarn \abs{ \gausswone_i })}_{=:T_{5,i}}  + \underbrace{ 2\xi_i\deltanu \abs{ \gausswone_i } \eta \gausswtwo_i }_{T_{6,i}} \vphantom{\sum} )
    \\
    =& \frac{1}{n}\sum_{i=1}^n \idvec[ 1- \xi_i\nu \abs{ \gausswone_i } - \eta \gausswtwo_i ] \left( T_{1,i} + T_{2,i} + T_{3,i} + T_{4,i} + T_{5,i} + T_{6,i} \right)
\end{align}

We now separately apply concentration inequalities to the different terms. We first need to control the difference of the term $\idvec[ 1- \xi_i\nu \abs{ \gausswone_i } - \eta \gausswtwo_i]$ to the term $\idvec[ 1- \xi_i\nubarn \abs{ \gausswone_i }]$. Denote by $F_n(x) = \frac{1}{n}\sum_{i=1}^n \idvec[\xi_i\abs{ \gausswone_i } \leq x ]$ and by $F(x) = \prob \left( \xi\abs{ \gausswone } \leq x \right) $. By  the Dvoretzky-Kiefer-Wolfowitz inequality (see \cite{massart}), we have that, with probability at least $ 1 - \frac{2}{d} $,
\begin{equation}
    \sup_{ x \in \RR } |F_n(x) - F(x) | \leq  \sqrt{ \frac{\log(d)}{2n} } =: \rho 
\end{equation}
for any $ n \geq 3 $. Furthermore, note that with probability $\geq 1-\frac{1}{d}$, we have that $\max_i \{ \abs{\gausswtwo_i}, \abs{\gausswone_i} \} = \gamma_m \lesssim \sqrt{\log(d)}$ (see e.g., \cite{boucheron_2012}).  Therefore, we can upper bound 
\begin{align}
    \frac{1}{n} \sum_{i=1}^n &\left( \idvec[ 1- \xi_i\nu \abs{\gausswone_i} - \eta \gausswtwo_i] - \idvec[ 1- \xi_i\nubarn \abs{\gausswone_i} ] \right)\\
    &\leq F_n \left( \frac{1}{\nubarn} ( 1 + \gamma_m (\abs{\deltanu} + \eta) ) \right) - F_n \left( \frac{1}{\nubarn} ( 1 - \gamma_m (\abs{\deltanu} + \eta) ) \right) 
    \\
    & \leq 2\rho + F \left( \frac{1}{\nubarn} (1+ \gamma_m (\abs{\deltanu} + \eta)) \right) - F \left( \frac{1}{\nubarn} (1- \gamma_m (\abs{\deltanu} + \eta)) \right)  \lesssim \rho + \gamma_m (|\deltanu| + \eta)
\end{align}
where we condition on the event where $\nubarn$ concentrates around the constant $\nubar$ in Lemma~\ref{lm:quantconcentration}.
Therefore, we get
\begin{align}
    &\left|\fnclas(\nu,\eta) - \frac{1}{n}\sum_{i=1}^n \idvec[ 1- \xi_i\nubarn \abs{\gausswone_i}] \left( T_{1,i} + T_{2,i} + T_{3,i} + T_{4,i} + T_{5,i} + T_{6,i} \right) \right| 
    \\
    =& \left|\frac{1}{n}\sum_{i=1}^n \left(\idvec[ 1- \xi_i\nu |\gausswone_i| - \eta \gausswtwo_i] - \idvec[ 1- \xi_i\nubarn |\gausswone_i|]  \right) \left( T_{1,i} + T_{2,i} + T_{3,i} + T_{4,i} + T_{5,i} + T_{6,i} \right) \right|
    \\ 
    \lesssim & (\rho + \gamma_m(|\deltanu|  + \eta))^3 
\end{align}

Thus, we only need to study $\frac{1}{n} \sum_{i=1}^n \idvec[ 1- \xi_i\nubarn \abs{\gausswone_i}] \left( T_{1,i} + T_{2,i} + T_{3,i} + T_{4,i} + T_{5,i} + T_{6,i} \right)$.
We now separately bound the different terms involving $T_{j,i}$ for every $j$ which all hold with probability at least $1- c_1 d^{-c_2}$ with universal constants $c_1,c_2>0$.
\begin{enumerate}
    \item[$T_{1,i}$] First, note that by definition, $ \frac{1}{n}\sum_{i=1}^n \idvec[ 1- \xi_i\nubarn \abs{\gausswone_i}] T_{1,i} = \fnstar$.
    \item[$T_{2,i}$] Recall that $ \frac{\cfnone}{2} = \frac{1}{n}\sum_{i=1}^n \idvec[ 1- \xi_i\nubarn \abs{\gausswone_i}] (\gausswone_i)^2$.
    \item[$T_{3,i}$] Recall that  $ \frac{\cfntwo}{2} = \frac{1}{n}\sum_{i=1}^n \idvec[ 1- \xi_i\nubarn \abs{\gausswone_i}] $.
        \begin{equation}
            \left| \frac{1}{n} \sum_{i=1}^n \idvec[ 1- \xi_i\nubarn \abs{\gausswone_i}] \eta^2 (\gausswtwo_i)^2 - \eta^2 \frac{\cfntwo}{2}  \right| 
            \leq 
            \eta^2 \left| \frac{1}{n}\sum_{i=1}^{k} ((\gausswtwo_i)^2 - 1) \right| 
           \lesssim \eta^2 \rho 
        \end{equation}
    with probability at least $ 1 - \frac{1}{d} $ over the draws of $\gausswtwo_i$. We used that $ \sum_{i=1}^n \idvec[ 1- \xi_i\nubarn \abs{\gausswone_i}] = k \leq n$. We apply the same reasoning in the following concentration inequalities.
    \item[$T_{4,i}$] Note that by definition of $\nubarn$ we have that
        \begin{equation}
            \frac{d}{d\nu}|_{\nubarn} f_n(\nu,0)  = \frac{1}{n} \sum_{i=1}^n \idvec[ 1- \xi_i\nubarn \abs{\gausswone_i}] 2(1- \xi_i\nubarn \abs{\gausswone_i})\xi_i \abs{\gausswone_i} = 0
        \end{equation}
    \item[$T_{5,i}$] We can bound the term involving $T_{5,i}$ by
        \begin{align}
            \left| \frac{1}{n} \sum_{i=1}^n \idvec[ 1- \xi_i\nubarn \abs{\gausswone_i}] 2\eta \gausswtwo_i \left( 1 - \xi_i\nubarn \abs{\gausswone_i} \right)  \right| 
            \leq
            2\eta \left| \frac{1}{n} \sum_{i=1}^n ( 1 + \nubarn \abs{\gausswone_i}) \gausswtwo_i 
            \right|
            \lesssim
            \eta \rho
        \end{align}
    \item[$T_{6,i}$]  We can bound the term involving $T_{6,i}$ by
        \begin{align}
            \left| \frac{1}{n} \sum_{i=1}^n \idvec[ 1- \xi_i\nubarn \abs{\gausswone_i}] 2 \xi_i \deltanu \abs{\gausswone_i} \eta \gausswtwo_i \right| 
            \leq
            2 \eta \abs{\deltanu}  \left| \frac{1}{n} \sum_{i=1}^n \gausswone_i \gausswtwo_i \right| \lesssim \eta \abs{\deltanu} \rho
        \end{align}
\end{enumerate}

\end{proof}

\subsection{Uniform convergence of $\fnclas$ to $\fclas$}
\label{apx:subsec:unif_conv}
We use the notation as described in the beginning of Appendix~\ref{apx:technicallemmas} and denote with $\Gausswone,\Gausswtwo \sim \Normal(0,1)$ independent Gaussian random variables and let $\xirv \vert \Gausswone \sim \probsigma(.;\Gausswone)$. 

Let $P$ be the distribution of $(\Gausswone,\Gausswtwo,\xirv)$ and let $P_n = \frac{1}{n} \sumin \delta_{\gausswone_i,\gausswtwo_i,\xi_i}$. Define $\gnueta(Z,\xi) = \possq{1-\xirv \abs{\Gausswone}\nu - \eta\Gausswtwo}$ with $Z = (\Gausswone, \Gausswtwo)$. Note that using this notation, we have $P\gnueta = \EE \gnueta(Z,\xi) = \fclas (\nu,\eta)$, with $f$ from Equation~\eqref{eq:defoff}, and $P_n \gnueta = \fnclas(\nu,\eta)$. Let $\GB = \left\{ \gnueta\ |\ \nu \leq \boundnuone, \eta \leq \boundetaone \right\}$ with positive constants $\boundetaone$ and $\boundnuone$. Moreover define:
\begin{align}
    \norm{ P_n - P }_{\GB} := \sup_{\gnueta \in \GB} \abs{(P_n - P)\gnueta}
\end{align}
We then obtain from an application of Theorem 4 in \cite{adamczak2008tail} the following uniform convergence result
\begin{proposition}
\label{prop:unifConv}
There exist positive universal constants $c_1,c_2,c_3>0$ such that
    \begin{align}
        \prob \left( \norm{ P_n - P }_{\GB} \leq \frac{c_1}{\sqrt{n}} + \epsilon \right) 
        \geq
        1 - c_2\exp \left( -c_3 \frac{ n \epsilon^2 }{\log n} \right)
    \end{align}
\end{proposition}

\begin{proof}

The proof follows from an application of the tail-bound result from Theorem 4 in \cite{adamczak2008tail}. More precisely, 

\begin{theorem}[Corollary of Theorem 4 in \cite{adamczak2008tail}]
\label{thm:unifconvmain}
For any $0<t<1$, $\delta>0$, $\alpha \in (0,1] $ there exists a constant $C = C ( \alpha, t, \delta )$ such that
\begin{align}
\label{eq:tailuniformconv}
    \prob \left( \norm{ P_n - P }_{\GB} \geq (1+t) \EE \norm{ P_n - P }_{\GB} + \epsilon \right)
    \leq
    \exp & \left( -\frac{\epsilon^2}{2(1+\delta)\sigma_{\GB}^2} \right) + 3 \exp \left( - \left( \frac{\epsilon}{C \psi_{\GB}} \right)^{\alpha} \right)
\end{align}
with variance term 
\begin{equation}
    \sigma_{\GB}^2 :=  \sup_{ \gnueta \in \GB } \frac{1}{n}\left( \EE \left[ \gnueta^2 \right] - \left( \EE \left[ \gnueta \right] \right)^2 \right),
\end{equation}
and $\psi_{\GB}$ defined as:
\begin{align}
    \psi_{\GB} = \orlicznorm{\max_{1\leq i\leq n} \sup_{\gnueta \in \GB} \frac{1}{n} \left\vert\gnueta(z_i,\xi_i) - \EE [\gnueta] \right\vert}
\end{align}
where $\orlicznorm{\cdot}$ denotes Orlicz norm. 
\end{theorem}

We choose $\alpha = 1$ and explicitly show that condition from Theorem 4 in \cite{adamczak2008tail} requiring finite Orlicz norms is indeed satisfied for this choice of $\alpha$.
We separate the proof into two steps, where in a first step we bound the $ \psi_{\GB}$ and then apply Theorem~\ref{thm:unifconvmain}. 

\paragraph{Step 1: Bounding $\psi_{\GB}$}

By the definition of Orlicz norms, $\psi_{\GB}$ is given by:
\begin{align}\label{eq_psi_GB_def}
    &\psi_{\GB} 
     =
    \orlicznorm{ \max_{1 \leq i \leq n} \sup_{\nu,\eta} \frac{1}{n} \abs{ \possq{ 1 - \xi_i \abs{\gausswone_i} \nu - \gausswtwo_i \eta } - \EE \left[ \possq{ 1 - \xi \abs{\gausswone} \nu - \gausswtwo \eta } \right] } } 
    \\ 
    & = \inf \left\{ \lambda>0:\ \EE \left[ \exp \left( \frac{1}{\lambda} \max_i \sup_{\nu,\eta} \frac{1}{n}
    \abs{ \possq{ 1 - \xi_i \abs{\gausswone_i} \nu - \gausswtwo_i \eta} - \EE \left[ \possq{ 1 - \xi \abs{\gausswone} \nu - \gausswtwo \eta } \right] } \right) - 1 \right] \leq 1\right\}
\end{align}
Applying the triangle inequality and using that $\nu,\eta$ are bounded by constants, we can bound the expectation from the expression above as:
\begin{align}
    & \EE \left[ \exp \left( \frac{1}{\lambda} \max_i \sup_{\nu,\eta} \frac{1}{n}
    \abs{ \possq{ 1 - \xi_i \abs{\gausswone_i} \nu - \gausswtwo_i \eta} - \EE \left[ \possq{ 1 - \xi \abs{\gausswone} \nu - \gausswtwo \eta } \right] } \right) \right]
    \\
    & \leq 
    \EE \left[ \exp \left(  \frac{1}{n\lambda} \max_i \sup_{\nu,\eta} 
    \possq{ 1 - \xi_i \abs{\gausswone_i} \nu -  \gausswtwo_i \eta} \right) \right]
     \exp \left( \frac{1}{n\lambda}\sup_{\nu,\eta} \EE \left[ \possq{ 1 - \xi \abs{\gausswone} \nu - \gausswtwo \eta } \right] \right) 
    \\
    & \leq
    \EE \left[ \exp \left(  \frac{B}{n\lambda}\zmax^2 \right) \right] \exp \left( \frac{B}{n\lambda} \right)  
\end{align}
with $B := 3\left( 1 + \boundnuone^2 + \boundetaone^2 \right)$ and $z_{(1)} = \max_{1\leq i\leq 2n} z_i = \max \left\{ \abs{\gausswone_1},\abs{\gausswtwo_1},\dots,\abs{\gausswone_n},\abs{\gausswtwo_n} \right\}$. Now we split the expectation from the above inequality into two terms:
\begin{align}
   \tilde{T}_1 = \EE \left[\onevec[\zmax< \sqrt{2 \log n}] \exp \left(  \frac{B}{n\lambda}\zmax^2 \right) \right]
   \leq
   \exp \left( \frac{2B\log n}{n\lambda} \right)
\end{align}
and
\begin{align}
    \tilde{T}_2 
    &=
    \EE \left[\onevec[\zmax\geq \sqrt{2 \log n}] \exp \left(  \frac{B}{n\lambda}\zmax^2 \right) \right]
    =
    4n \int_{\sqrt{2 \log n}}^{\infty} \int_{-z_1}^{z_1} \cdots \int_{-z_1}^{z_1} \exp \left( \frac{B}{n\lambda}  z_1^2 \right) \prod_{i=1}^{2n} \frac{\exp(-\frac{1}{2}z_i^2)}{\sqrt{2\pi}} dz_i
    \\
    &\leq
    4n \int_{\sqrt{2 \log n}}^{\infty} \frac{1}{\sqrt{2\pi}} \exp \left( -\frac{1}{2} z_1^2 \right) \exp \left( \frac{B}{n\lambda} z_1^2 \right) dz_1 
    \lesssim
    \frac{n^{\frac{2B}{n\lambda}}}{\sqrt{\log n}(1-\frac{2B}{n\lambda})}
\end{align}
where we assumed that $\lambda > \frac{2B}{n}$. Now choosing $\lambda = c_{\lambda} \frac{\log n}{n}$ with a positive constant $c_{\lambda}$ sufficiently large, we find that the condition in Equality \eqref{eq_psi_GB_def} is satisfied for this $\lambda$, which implies that $\psi_{\GB} \leq c_{\lambda}\frac{\log n}{n}$.
\paragraph{Step 2: Proof of the statement}
To apply Theorem~\ref{thm:unifconvmain}, we also need to bound the variance $\sigma_{\GB}^2$ and use that $\EE \norm{P_n-P}_{\GB} \leq 2 \rademacher(\GB)$  where $ \rademacher(\GB)$ is the Rademacher complexity given by $\rademacher(\GB) = \EE \left[ \sup_{\gnueta\in \GB} \left|\frac{1}{n}\sumin \epsilon_i g_{\eta,v}(z_i,\xi_i) \right| \right]$. 
The bound on the variance follows from a straightforward calculation:
\begin{align}
    \sigma_{\GB}^2
    \leq 
    \frac{1}{n} \sup_{ \gnueta \in \GB } \EE \left[ \gnueta^2 \right]
    \leq
    \frac{1}{n} c_{\sigma_{\GB}} \left( 1 + \boundnuone^4 + \boundetaone^4 \right)
\end{align}
which holds for some positive universal constant $c_{\sigma_{\GB}}>0$. 

Next, note that we can upper bound the Rademacher complexity using the triangle inequality and the fact that $(\cdot)_+$ is $1$-Lipschitz:
\begin{align}
     \rademacher(\GB) &\leq 
    \EE \left[\sup_{\nu\leq \boundnuone,\eta\leq \boundetaone}\left| \frac{1}{n}\sumin \epsilon_i (1-\xi_i \abs{\gausswone_i}\nu - \eta\gausswtwo_i)^2\right| \right] 
    \\ 
    &\lesssim \EE \left[ \sup_{\nu\leq \boundnuone,\eta\leq \boundetaone}\left| \frac{1}{n}\sumin \epsilon_i\right| \right] +  \EE \left[ \sup_{\nu\leq \boundnuone,\eta\leq \boundetaone}\left| \frac{1}{n}\sumin \epsilon_i (\xi_i \abs{\gausswone_i}\nu)^2\right| \right] \\
    &+  \EE \left[ \sup_{\nu\leq \boundnuone,\eta\leq \boundetaone}\left| \frac{1}{n}\sumin \epsilon_i ( \eta\gausswtwo_i)^2\right| \right] 
    \lesssim \frac{1}{\sqrt{n}}
\end{align}
where the last line follows from applying the standard concentration results. Thus, we obtain:
    \begin{align}
        \prob \left( \norm{ P_n - P }_{\GB} \geq 2(1+t) \mathcal{R}_{\GB} + \epsilon \right) 
        \leq
        \exp \left( - c_2 n \epsilon^2 \right) + 3\exp \left( -c_3 \frac{ n \epsilon }{\log n} \right)
    \end{align}
with $c_2^{-1} = 2(1+\delta)c_{\sigma_{\GB}} B^2$ and $c_3^{-1} = C c_{\lambda}$, which concludes the proof. 

\end{proof}

\section{Proof of Lemma~\ref{lm:dualnormconc}}
\label{apx:proofdualnormconc}
\label{apx:concentrationdual}

First note that by taking subgradients we directly obtain $\langle \wgrad, H\rangle = \norm{H}_q$ and 
$\norm{\wgrad}_2^2 = \left(\frac{\norm{H}_{2q/p}}{\norm{H}_q}\right)^{2q/p}$. 

Recall the definitions $\dmoment{q} = \EE \|H\|_{q}$ and
$\mutild := \left( \frac{\dmoment{2q/p}}{\dmoment{q}} \right)^{2q/p}$. We use the following well-known concentration result on $\ell_q$-norms to control these quantities.
 \begin{theorem}[Corollary of Theorem 1 \cite{paouris_2017}]
 \label{cor:hqbound1}
For all sufficiently large $d$ and any $2<q <c_1\log(d)$ and $\epsilon \in (0,1)$, we have: 
\begin{equation}
\label{eq:paouris}
    \prob(|\|H\|_q - \dmoment{q} | \geq \epsilon \EE \|H\|_q ) \leq c_2 \exp(-c_3 \min(q^2 2^{-q} \epsilon^2 d, (\epsilon d )^{2/q}))
\end{equation}
where $c_1\in (0,1),c_2,c_3>0$ are universal constants.
\end{theorem}

The proof of the lemma follows by showing that the choice of $\epsilon$ leads to sufficiently small terms in the exponent on the RHS of Equation~\eqref{eq:paouris}. More precisely, we show that there  exists a universal constant $c>0$, such that for all $n,d$ sufficiently large with $d \geq n$ and  $2\leq  q \leq c \log\log(d)$, we have that
\begin{equation}
    \min(q^2 2^{-q} \epsilon ^2 d, (\epsilon  d )^{2/q}) \geq \log(d), \label{eq:proofconctempeq}
\end{equation}
\label{claim_hq_proof}
and similarly for $2q/p$ instead of $q$. 

Equation~\eqref{eq:dualnormconc1} then follows directly from the claim and Theorem~\ref{cor:hqbound1}. To see Equation~\eqref{eq:dualnormconc2}, we have that with probability $\geq 1-c_1 d^{-c_2}$,
\begin{align}
     \norm{H}_{2q/p}^{2q/p} = \dmoment{2q/p}^{2q/p} \left(1+ O(\epsilon )\right)^{2q/p}
     = \dmoment{2q/p}^{2q/p} \left(1+  O( q\epsilon ))\right),
\end{align}
where we have used in the last equality that $q\lesssim \log\log(d)$ and that $\epsilon  = O(\log^{-1}(d))$ (by assumption on $d,n$ and $\epsilon$). Furthermore, the proof follows from the fact that a similar argument also applies to $\norm{H}_{q}^{2q/p}$. 

Thus, the only thing left to show is Equation~\eqref{eq:proofconctempeq}. 
We separately check the different conditions in Theorem~\ref{cor:hqbound1}. First, note that it is sufficient to only consider the case where $q = c \log\log(d)$. We separate the proof into two cases where   $n^2 \geq d \log^{c_2}(d)$  and $n^2 \leq d \log^{c_2}(d)$ with constant $c_2>0$. 
\paragraph{Case  $n^2 \geq d \log^{c_2}(d) \implies \epsilon  =\frac{n}{d}$} 
Looking only at the first term on the RHS in Equation~\eqref{eq:proofconctempeq}, we have 
\begin{align}
    \log(q^2 2^{-q} \epsilon ^2 d) = 2\log(q) - q\log(2) + 2\log(n) - \log(d) \geq \log\log(d),
\end{align}
Furthermore, looking only at the second term on the RHS in  Equation~\eqref{eq:proofconctempeq}, we have 
\begin{align}
   \frac{2}{q} \log(\epsilon d)=  \frac{2}{q}\log(n) \geq \log\log(d),
\end{align}
which holds also under weaker assumptions on $n,d$. Hence the proof for the first case is complete. 

\paragraph{Case $n^2 \leq d \log^{c_2}(d) \implies \epsilon  =\frac{\log^{c_1}(d)}{n}$.} 
From straight forward calculation, we note that there exists a universal constant  $c_1>0$ only depending on $c_2$ such that 
\begin{align}
    \log(q^2 2^{-q} \epsilon ^2 d) = 2\log(q) - q\log(2) + \log(d) - 2\log(n) + 2c_1\log\log(d) \geq \log\log(d),
\end{align}
for any $n$ such that $n^2 \leq d \log^{c_2}(d)$. 
Furthermore, under the same assumption on $d,n$, we have
\begin{align}
   \frac{2}{q} \log(\epsilon d)=  \frac{2}{q} (\log(d)-\log(n) +c_1 \log\log(d)) \geq \log\log(d).
\end{align}

Finally the expectations $\mutild, \dmoment{q}$ are directly obtained using the following well-known result on the expectations of the $\ell_q$ norms of Gaussian vectors
\begin{proposition}[Proposition 2.4 \cite{paouris_2017}, \cite{schechtman_89}]
\label{cor:hqbound2}
For all $q \leq \logjone{d}$, we have that $\dmoment{q} \asymp \sqrt{q} d^{1/q}$.
\end{proposition}
This concludes the proof.



\end{document}